\documentclass{article}

\usepackage{cite}
\usepackage[english]{babel}

\usepackage{amsthm}
\usepackage{graphicx}
\usepackage{amsmath}
\usepackage{amsthm}
\usepackage{amssymb}
\usepackage{algorithmic}
\usepackage[ruled,vlined]{algorithm2e}
\usepackage{braket}
\usepackage{comment}
\usepackage{mathtools}
\usepackage[utf8]{inputenc}
\usepackage[mathscr]{eucal}
\usepackage{fancyhdr}
\usepackage{fancybox}
\usepackage{xcolor}
\usepackage{longtable}
\usepackage{dsfont}
\usepackage[colorlinks=true, allcolors=blue]{hyperref}

\usepackage{pifont}
\newcommand{\xmark}{\ding{55}}%

\usepackage{caption}
\usepackage{subcaption}

\newcommand\balpha{\boldsymbol{\alpha}}

\newcommand{\defeq}{\vcentcolon=}

\def\R{\mathbb{R}}

\def\x{\mathbf{x}}
\def\m{\mathbf{m}}
\def\balpha{\boldsymbol{\alpha}}
\def\y{\mathbf{y}}
\def\bt{\mathbf{t}}
\def\btheta{\boldsymbol{\theta}}
\def\u{\mathbf{u}}
\def\v{\mathbf{v}}
\def\U{\mathbf{U}}
\def\X{\mathbf{X}}
\def\Y{\mathbf{Y}}
\def\Y{\mathbf{Y}}
\def\K{\mathbf{K}}
\def\O{\mathcal{O}}

\newtheorem{theorem}{Theorem}[section]
\newtheorem{proposition}[theorem]{Proposition}
\newtheorem{lemma}[theorem]{Lemma}
\newtheorem{corollary}[theorem]{Corollary}
\newtheorem{remark}[theorem]{Remark}

\begin{document}

\begin{center}
{\Large Gaussian Processes on Distributions based on Regularized Optimal Transport} \\
\vspace{0.5cm} 
Fran\c cois Bachoc$^{1,2}$\\
 Louis B\'ethune$^{1,3}$\\
Alberto Gonzalez-Sanz$^{1,2}$\\
Jean-Michel Loubes$^{1,2}$ 
\\
\vspace{0.2cm} 
$^1$ Université Paul Sabatier \\
$^2$ Institut de Mathématiques de Toulouse \\
$^3$ Institut de Recherche en Informatique de Toulouse
\end{center}

\begin{abstract}
    We present a novel kernel  over the space of probability measures based on the dual formulation of optimal regularized transport. We propose an Hilbertian embedding of the space of probabilities using their Sinkhorn potentials,  which are solutions of the dual entropic relaxed optimal transport between the probabilities and a reference measure $\mathcal{U}$. We prove that this construction enables to obtain a valid kernel, by using the Hilbert norms. We prove that the kernel enjoys theoretical properties such as   universality and some invariances, while still being computationally feasible. Moreover we provide theoretical guarantees on the behaviour of a Gaussian process based on this kernel. The empirical performances are compared with other traditional choices of kernels for processes indexed on distributions.
\end{abstract}


\section{Introduction}

\textbf{Context: Gaussian processes and kernels indexed by distributions.}
Gaussian process (GP) models are widely used in fields such as geostatistics, computer code experiments and machine learning. We refer to \cite{rasmussen2006gaussian} for general references.  They consist in modeling an unknown function as a realization of a GP, and hence correspond to a functional Bayesian framework. For instance, in computer experiments, the input points of the function are simulation parameters and the output values are quantities of interest obtained from the simulations. GPs rely on the definition of a covariance function that characterises the correlations between values of the process at different observation points.  \\
In this paper we consider GPs indexed by distributions. Learning functions defined on distributions has gained a special interest over the last decade in the machine learning literature,
see for instance \cite{poczos2013distribution}. Distribution-valued inputs are commonly used to describe complex objects such as images, shapes or media as described for instance in \cite{glaunes2004diffeomorphic},  \cite{muandet2012learning}, \cite{ginsbourger2016design} or \cite{szabo2016learning}.
The construction of a kernel for these inputs requires a notion of similitude between the probability distributions. Many methods have been considered to provide kernels for distributions, from the mere extraction  of parametric features, such as the mean or higher moments, to the well used Maximum Mean Discrepancy method \cite{gretton2012kernel}. \\  
\vskip .1in
\textbf{Context: optimal transport for kernels.} On this topic, optimal transport (OT) has imposed itself has a prominent method for comparing or analyzing distributions. Previous works in this direction are, in one dimension, \cite{bachoc2017gaussian} and \cite{thi2021distribution}, where a kernel directly based on the quadratic difference between the quantiles, which yields (on the real line) the quadratic Wasserstein distance, is proposed. In several dimensions, a quite  natural generalisation uses the quadratic norm between the multidimensional transport maps between the probabilities and a reference measure, see \cite{bachoc2020gaussian} and \cite{moosmuller2020linear}. Even though, with a good choice of the reference measure, the generated kernels are translation invariant (see \cite{Hallin2020DistributionAQ,DELBARRIO2020104671,Barrio2022NonparametricMC}), for machine learning purposes the computation of the transport map (between continuous measures) is rather complicated and depends highly on the dimension, see \cite{peyre2019computational}. Moreover, even if the transport maps exist almost surely (see \cite{McCann1995ExistenceAU}) for a suitable choice of the reference distribution, their continuity---and therefore their approximations from empirical settings--require at least upper-lower bounded densities and the convex supports of the target measures, see \cite{Figalli}. 
Regarding GPs based on OT, the high complexity of the transport problem makes it so that their continuity properties are not much studied especially in multi-dimension. 
As a simplification, \cite{kolouri2018sliced} proposed the use of a slice-Wasserstein kernel. The idea is to reduce the problem by projecting on the different directions generated by a (uniform) discretization of the unit sphere $\mathbb{S}_{d-1}$, and then integrating w.r.t. the uniform measure on $\mathbb{S}_{d-1}$. This avoids completely the curse of the dimension, but this does not discriminate quite well between non convex domains. The corresponding universality properties are studied in the recent work of~\cite{meunier2022distribution}.  \vskip .1in
\textbf{Contributions.}
While the initial formulations of OT yield computational challenges, subsequent regularized versions have provided valuable trade-offs between richness and tractability. In this work, we provide a kernel based on regularized OT. We prove that  the norm  between the  potentials derived from entropy relaxation of Wasserstein distances, see \cite{cuturi2013sinkhorn}, provides a natural embedding for the distributions and can be used to construct a valid kernel. Much work on the properties of the potentials has been carried out in particular  \cite{MenaWeed,Barrio2022AnIC,gonzalez2022weak} but few results exist taking advantage of the natural embedding the potentials provide.
{\bf 1)} Our contribution is first to propose a novel valid and universal kernel based on Sinkhorn's dual potentials when considering the regularized transport towards a reference measure. {\bf 2)} We then propose statistical guarantees for this kernel by studying the properties of its empirical counterpart as well as invariance properties w.r.t. the choice of the reference. {\bf 3)} We study the theoretical properties of  the corresponding GP, especially the existence of a continuous version.  Feasible computations through Sinkhorn's algorithm enable to study the prediction performance of the kernel. {\bf 4)} We provide publically available code, together with simulations and real datasets where our kernel competes favorably with state of the art methods: it yields a very similar accuracy while providing a computational seed-up of order up to $100$.\\
\vskip .1in
\textbf{Outline.}
Section~\ref{sec:defbasics} is devoted to providing some definitions  and notations related to Sinkhorn's transport methods. In Section~\ref{sec:kerneldef} we define and study the  kernel based on the potentials, while Section~\ref{s:GP} studies the GP with this kernel as covariance operator. Implementation and experiments are discussed in Section~\ref{s:implementation}. The proofs and complementary content are postponed to the Appendix.


\section{Definitions and basic properties of Sinkhorn distance}\label{sec:defbasics}
\subsection{General definitions and notations} \label{subsection:definition:notation}

We let $\mathcal{P}(A)$ be the set of probability measures on a general set $A \subset \mathbb{R}^d$.
When $A$ is compact and for $s >0$, we let $\mathcal{C}^s(A)$ be the space of functions $f: A \to \mathbb{R}$ that are $\lfloor s \rfloor$ times differentiable, with $\lfloor . \rfloor$ the integer part, with $\|f \|_{\mathcal{C}^s(A)} < \infty$ where
\begin{equation} \label{eq:norm:holder:alpha}
    \|f \|_{\mathcal{C}^s(A)} := \sum_{i=0}^{\lfloor s \rfloor}\sum_{|\balpha|= i}\|D^{\balpha} f \|_{\infty}.
    \end{equation}
    Above $\balpha = (\alpha_1, \ldots , \alpha_d) \in \mathbb{N}^d$ with $\sum_{j=1}^d \alpha_j = i$ and $D^{\balpha} = \partial^i / \partial _{x_1}^{\alpha_1} \cdots \partial _{x_d}^{\alpha_d}$. The space $\mathcal{C}^s(A)$ is endowed with the norm $\|\cdot \|_{\mathcal{C}^s(A)}$.
A probability     $\mathrm{P} \in \mathcal{P}(A)$ belongs also to the topological dual space of $\mathcal{C}^s(A)$. A distance between two measures $\mathrm{P},\mathrm{Q} \in \mathcal{P}(A)$ can be defined as
\begin{equation} \label{eq:norme:P:Q:C:s:un}
    \|{\rm P}-{\rm Q}\|_s:=\sup_{f\in \mathcal{C}^{s}(A),\ \| f\|_{\mathcal{C}^s(A)}\leq 1}\int f(\x)(d \mathrm{P}(\x)-d \mathrm{Q}(\x)).
\end{equation}
    
We let $\ell_d$ be d-dimensional Lebesgue measure. For $p>0$ and for $\mathrm{P}
\in \mathcal{P}(A)$, we let $L^p(\mathrm{P})$ be the set of functions $f: A \to \mathbb{R}$ such that $\|f\|_{L^p(\mathrm{P})}^p := \int_{A} |f(\x)|^p d\mathrm{P}(\x) < \infty$.

We use the abreviations ``a.s.'' for ``almost surely'' and ``a.e.'' for ``almost everywhere''.
For a probability measure $\mathrm{P}$ on $A$, we let $\mathrm{supp}(\mathrm{P})$ be its topological support (the smallest closed set with $\mathrm{P}$-probability one).
For two sets $A$ and $B$, for a probability measure $\mathrm{P}$ on $A$, and for $T: A \to B$, we let $T \sharp A$ be the probability measure of $T(\boldsymbol{X})$ where $\boldsymbol{X}$ is a random vector with law $\mathrm{P}$. For two probability distributions $\mathrm{P}$ and $\mathrm{Q}$ on $A$, we write $\mathrm{P} \ll \mathrm{Q}$ when $\mathrm{P}$ is absolutely continuous w.r.t. $\mathrm{Q}$ and in this case we write $d \mathrm{P} / d\mathrm{Q}$ for the density of $\mathrm{P}$ w.r.t. $\mathrm{Q}$.
A random vector $\mathbf{V}$ on $A \subset \mathbb{R}^d$ is said to be sub-Gaussian if there is $\sigma^2 < \infty$ such that $\mathbb{E}( \exp( s \mathbf{u}^\top \mathbf{V} ) ) \leq \exp(\sigma^2 s^2 / 2)$ for any $\mathbf{u} \in \mathbb{R}^d$, $\|\mathbf{u}\| = 1$ and $s \in \mathbb{R}$.
We let $\mathcal{P}_{SG}(A)$ be the set of sub-Gaussian probability measures on $A$. For $\x \in A$ we let $\delta_{\x}$ be the Dirac probability measure at $\x$.

For a set $E$, a function $k : E \times E \to \mathbb{R}$ is said to be positive definite when for any $x_1,\ldots,x_n \in E$, $\alpha_1,\ldots,\alpha_n \in \mathbb{R}$, $\sum_{i,j=1}^n \alpha_i \alpha_j k(x_i, x_j) \geq 0$. The function is said to be strictly positive definite if in addition the sum is strictly positive when $x_1,\ldots,x_n$ are two-by-two distinct and not all $\alpha_1,\ldots,\alpha_n$ are zero.

For $A \subset \mathbb{R}^d$ we let $\mathrm{diam}(A) = \sup \{ \| \x-\y \| ; \x,\y \in A \}$. For $t \in \mathbb{R}$, we let $\left\lceil t \right\rceil$ be the smallest integer larger or equal to $t$. For two column vectors ${\bf x},{\bf y}$, we let $\langle {\bf x},{\bf y} \rangle = {\bf x}^\top {\bf y}$ be their scalar product.

\subsection{Regularized optimal transport}
We consider an input space $\Omega \subset \mathbb{R}^d$ that is fixed throughout the paper. For some of the results of the paper, $\Omega$ will be assumed to be compact, while for others, we can make the weaker assumption to consider sub-Gaussian measures on $\Omega$ (that is not necessarily bounded). 
Let ${\rm P}$, ${\rm Q}$ be probabilities on $\Omega$ and set  $\Pi({\rm P},{\rm Q})$ the set of probability measures $\pi \in \mathcal{P}(\Omega \times \Omega)$ with marginals ${\rm P}$ and ${\rm Q}$, i.e. 
  for all $A,B$ measurable sets
  \begin{equation}
\pi(A \times \Omega) = {\rm P}(A),
~ ~ ~ ~
\pi(\Omega  \times B) = {\rm Q}(B).
  \end{equation}

The OT problem amounts to solve the optimization problem  (see \cite{kantorovich1942translocation})
\begin{equation} \label{kant}
\mathcal{T}_c({\rm P}, {\rm Q}):=\min_{\pi \in\Pi({\rm P},{\rm Q})} \int c(\x,\y) d\pi(\x,\y),
\end{equation}
with a continuous cost $c: \Omega \times \Omega \to [0,\infty)$. It is well known (see eg. \cite{Villani2003})  that $\mathcal{W}_p({\rm P}, {\rm Q}) :=\left(\mathcal{T}_{\|\cdot \|^p}({\rm P}, {\rm Q})\right)^{\frac{1}{p}}$---the value of \eqref{kant} for a potential cost $(\x,\y)\mapsto\|\x-\y \|^p$, for $p\geq 1$---defines a distance on the space of probabilities with finite moments of order $p$. This distance is called the Wasserstein distance.  

In this paper we will consider the quadratic cost $c(\x,\y)=\|\x-\y\|^2$. In this setting, when at least one distribution $P$ is absolutely continuous w.r.t. Lebesgue measure, then there exists a $P$-a.e. unique map $T : \Omega \to \Omega$ such that $T \sharp P= Q$,  and  $\mathcal{W}_2 (P, Q)^2 = \int_{\Omega} \|T(\x) - \x\|^2 d {\rm P}(\x)$. 
Moreover, there exists a lower semi-continuous convex function $\varphi$  such that
$T = \nabla \varphi$ ${\rm P}$-a.e., with $\nabla$ the gradient operator, and $T$ is the only map of this type pushing forward ${\rm P}$ to ${\rm Q}$, up to a
${\rm P}$-negligible modification.  This  theorem above is commonly referred to as Brenier's theorem in \cite{brenier1991polar}. Note that   a similar statement was established earlier independently in a
probabilistic framework in \cite{cuesta1989notes}. \\
This result enables to define a natural Hilbertian embedding of the distributions in $\mathcal{P}(\Omega)$ by considering the distance between the transport maps towards a common reference distribution. This framework has been used in \cite{bachoc2020gaussian} to provide kernels on distributions. Yet such kernels have the drawback of being difficult to compute, preventing their use for large or high-dimensional data sets.\\
Indeed, computing the OT \eqref{kant} turns out to be  computationally difficult. In the discrete case, different algorithms have been proposed such as the Hungarian algorithm \cite{kuhn1955hungarian}, the simplex algorithm \cite{luenberger1984linear} or others versions using interior points algorithms \cite{orlin1988faster}. The complexity of these methods is at worst of order $O(n^3 \log(n))$ for two discrete distributions with equal size $n$. Hence \cite{bachoc2020gaussian} and many statistical methods based on OT suffer from this drawback.  \vskip .1in
To overcome this issue, regularization methods have been proposed to approximate the OT problem by adding a penalty. The seminal paper by \cite{cuturi2013sinkhorn} provides the description of the Sinkhorn algorithm to regularize  OT by using an entropy penalty. \vskip .1in
The relative entropy between two probability measures ${\rm \alpha}, {\rm \beta}$ on $\Omega$, is defined as 
\[
H({\rm \alpha}| {\rm \beta})=\int_{\Omega} \log(\frac{d{\rm \alpha}}{d{\rm \beta}}(\x))d{\rm \alpha}(\x)
\]
if ${\rm \alpha}\ll {\rm \beta}$ and $|\log( d{\rm \alpha} / d{\rm \beta} )| \in L^1({\rm \beta})$, and $+\infty$ otherwise. Set $\epsilon>0$. Then the entropy regularized version of the OT problem is defined as
\begin{equation}\label{kanto_entrop}
\begin{aligned}
        S_{\epsilon}({\rm P},{\rm Q})\defeq \min_{\pi\in \Pi({\rm P},{\rm Q})} &\int_{\Omega\times \Omega} \frac{1}{2}\|\x-\y\|^2 d\pi(\x,\y)\\
        &+\epsilon H(\pi | {\rm P}\times {\rm Q}),
\end{aligned}
\end{equation}
with ${\rm P}\times {\rm Q}$ the product measure.
The entropy term $H$ modifies the linear term in classical OT (the quadratic transportation cost)
to produce a strictly convex functional. The parameter $\epsilon$ balances the trade-off between the classical OT problem ($\epsilon=0$) and the influence of the regularizing penalty.

The minimization of \eqref{kanto_entrop} is achieved using Sinkhorn algorithm. We refer to \cite{peyre2019computational} and references therein for more details.
The introduction of the Sinkhorn divergence enables to obtain an $\varepsilon$-approximation of the OT distance which can be computed, as pointed out in \cite{altschuler2017near},  with a  complexity of algorithm  of order $O(\frac{n^2}{\varepsilon^3})$, hence in a much faster way than the original OT problem. Several toolboxes have been developed to compute regularized OT such among others as \cite{flamary2017pot} for Python, \cite{klatt2017package} for R. \vskip .1in

Contrary to (unregularized) OT, Sinkhorn OT does not provide transport maps, which would in turn provide a Hilbertian embedding. Hence, we consider the dual formulation of \eqref{kanto_entrop} pointed out in \cite{genevay2019phd}:
\begin{equation}\label{dual_entrop}
\begin{aligned}
&  S_{\epsilon}({\rm P},{\rm Q})= \!  \sup_{f\in L^1({\rm P}),g\in L^1({\rm Q})} \!\int_{\Omega}  f(\x)
  d {\rm P}(\x)+ \!\int_{\Omega} g(\y) d {\rm Q}(\y)\\ 
 & - \!\epsilon \int_{\Omega \times \Omega} e^{\frac{1}{\epsilon} \left({f(\x)+g(\y)- \frac{1}{2}\|\x-\y\|^2}\right)} d {\rm P}(\x)dQ(\y)+\epsilon.
 \end{aligned}
\end{equation}
Note that this formulation is a convex relaxation of the duality of the usual OT. Both primal and dual problems have solutions if ${\rm P}$ and ${\rm Q}$ have finite second moments. \vskip .1in
Let $\pi$ be the solution to \eqref{kanto_entrop} which will be denoted as  the \emph{optimal entropic plan}. Let $(f, g)$ be  the solution to \eqref{dual_entrop}, which will be denoted as the \textit{optimal entropic potentials}. For ${\rm P}, {\rm Q} \in \mathcal{P}_{SG}(\Omega)$, both quantities can be related using the formula
\begin{equation} \label{key}
\frac{d \pi}{d {\rm P} d {\rm Q}}= \exp\left(-\frac{1}{\varepsilon}\left(f(\x)+g(\y)-\frac{1}{2}\|\x-\y\|^2 \right)\right).
\end{equation}
A consequence of this relation is that we have the \emph{optimality conditions}
\begin{align} 
\int e^{\frac{1}{ \epsilon} (f(\x) + g(\y) - \frac{1}{2}\|\x-\y\|^2)} d {\rm P}(\x) & =1, \quad \forall \y \in \Omega,
\label{eq:optimality:conditions:un} \\
\int e^{\frac 1 \epsilon (f(\x) + g(\y) - \frac{1}{2}\|\x-\y\|^2)} d {\rm Q}(\y) & =1, \quad \forall \x \in \Omega.
\label{eq:optimality:conditions:deux}
\end{align}
\section{A Kernel based on  regularized optimal transport}\label{sec:kerneldef}
\subsection{Construction of positive definite kernels} \label{subsection:construction:positive:kernel}

    Consider  a reference measure $\mathcal{U}$ on $\Omega$. For two distributions ${\rm P}$ and ${\rm Q}$, consider the two regularized OTs respectively between ${\rm P}$ and $\mathcal{U}$ and between ${\rm Q}$ and $\mathcal{U}$. Let 
    $\pi^{\rm P}_{\mathcal{U}}$ and $\pi^{\rm Q}_{\mathcal{U}}$ be the optimal entropic plans and $(f^{\rm P}_{\mathcal{U}},g^{\rm P}_{\mathcal{U}})$ and $(f^{\rm Q}_{\mathcal{U}},g^{\rm Q}_{\mathcal{U}})$ the optimal entropic potentials: 
    
    \begin{align}
    \frac{d \pi^{\rm P}_{\mathcal{U}}}{dPd \mathcal{U}} & = \exp \left( -\frac{1}{\varepsilon}\left(f^{\rm P}_{\mathcal{U}}(\x)+g^{\rm P}_{\mathcal{U}}(\y)-\frac{1}{2}\|\x-\y\|^2 \right) \right)  \label{potP} \\
    \frac{d \pi^{\rm Q}_{\mathcal{U}}}{dQd\mathcal{U}} & = \exp \left(-\frac{1}{\varepsilon}\left(f^{\rm Q}_{\mathcal{U}}(\x)+g^{\rm Q}_{\mathcal{U}}(\y)-\frac{1}{2}\|\x-\y\|^2 \right) \right).\label{potQ}
    \end{align}
    Our aim is to use the distance $\|g^{\rm P}_{\mathcal{U}}-g^{\rm Q}_{\mathcal{U}}\|_{L^2(\mathcal{U})}$ to build Sinkhorn kernels. Note first that the uniqueness of Sinkhorn potentials holds up to additive constants. To obtain uniqueness, from now on, we will define $g^{\rm P}_{\mathcal{U}}$ as the unique centered (w.r.t. to $\mathcal{U}$) potential.
    This implies that $g^{\rm P}_{\mathcal{U}}=g^{\rm P}_{\mathcal{U}}-\mathbb{E}(g^{\rm P}_{\mathcal{U}}(U))$, which yields the following equality 
    $$ \operatorname{Var}_{\U\sim\mathcal{U}}(g^{\rm P}_{\mathcal{U}}(\U)-g^{\rm Q}_{\mathcal{U}}(\U))=\|g^{\rm P}_{\mathcal{U}}-g^{\rm Q}_{\mathcal{U}}\|_{L^2(\mathcal{U})}^2.$$

Then,  a function $f :
    [0, \infty) \to \mathbb{R}$ is said to be completely monotone if it is $C^{\infty}$ on $(0, \infty)$, continuous at $0$ and
    satisfies $(-1)^\ell f^{( \ell )}(r) \geq 0$  for $r > 0$ and  $\ell \in  \mathbb{N}$.
    Let $F : [0, \infty) \to \mathbb{R}$ be continuous. \\
    The following theorem provides the kernel construction and its validity (positive-definiteness).
    
    \begin{theorem} \label{theorem:F:kernel}
    Let $K: \mathcal{P}_{SG}(\Omega) \times \mathcal{P}_{SG}(\Omega)  \to  \mathbb{R}$ be 
      the function defined as
     \begin{align}
         \begin{split}\label{eq:kernel}
            ( \mathrm{P}, \mathrm{Q})&\mapsto  K({\rm P},{\rm Q})=F(\|g^{\rm P}_{\mathcal{U}}-g^{\rm Q}_{\mathcal{U}}\|_{L^2(\mathcal{U})}),
         \end{split}
     \end{align}
     for some $\mathcal{U}\in  \mathcal{P}_{SG}(\Omega)$.
    Then the two following conditions are sufficient conditions for $K$ to be a positive definite kernel on $\mathcal{P}_{SG}(\Omega)$. 
    \begin{enumerate}
        \item $F(
    \sqrt{.})$ is completely monotone on $[0, \infty)$. 
    \item There exists a finite nonnegative Borel measure ${\rm \nu}$ on $[0, \infty)$ such that for $t \geq 0$
    $F(t) = \int_{0}^{\infty}
    e^{ - u t^2}
   d {\rm \nu} (u )$.
    \end{enumerate}
    \end{theorem}
    Remark that the quantity $\|g^{\rm P}_{\mathcal{U}}-g^{\rm Q}_{\mathcal{U}}\|_{L^2(\mathcal{U})}$ in Theorem~\ref{theorem:F:kernel} is finite via \cite[Proposition 1]{MenaWeed}.
    
    Examples of functions $F$ for which the assumptions of Theorem \ref{theorem:F:kernel} are satisfied are the well-known square exponential, power exponential and Mat\'ern covariance functions, see \cite{bachoc2020gaussian} and the references therein.

    The following proposition bounds the $L^2(\mathcal{U})$ distance between the potentials as a function of the distance between the distributions. 
\begin{proposition}\label{Lemma:cuadratic}
Let $s \in \mathbb{N}$.
Assume that $\Omega$ is compact and let ${\rm P},{\rm Q}\in \mathcal{P}(\Omega)$. Then there exists a constant $c_{d}$, depending on the dimension, such that 
\begin{multline*}
    \|g^{\rm P}_{\mathcal{U}}-g^{\rm Q}_{\mathcal{U}}\|_{L^2(\mathcal{U})}
    \leq c_{d}\operatorname{diam}({\Omega})^{s} e^{ \frac{19}{2}\operatorname{diam}({\Omega})^2}\|P-{\rm Q}\|_s.
\end{multline*}
\end{proposition}
Note that the previous bound is still valid if we replace $\|P-{\rm Q}\|_s$ by $\mathcal{W}_1({\rm P},{\rm Q})$. This remark   
    follows directly from Kantorovich's duality, see Theorem~1.14 in \cite{Villani2003}.

     The following proposition guarantees that the entropic potentials $g^{\rm P}_{\mathcal{U}}$ and $g^{\rm Q}_{\mathcal{U}}$ can be used to characterize the distributions ${\rm P}$ and ${\rm Q}$. It also guarantees that our suggested kernel is not only positive definite but also strictly positive definite.    
    
    \begin{proposition}\label{Lemma:uniqueness}
    Let ${\rm P},{\rm Q}, \mathcal{U} \in \mathcal{P}_{SG}(\Omega)$.
    The potentials $g^{\rm P}_{\mathcal{U}}(\u)$ and $g^{\rm Q}_{\mathcal{U}}(\u)$ can be extended continuously with \eqref{eq:optimality:conditions:deux} for $\u \in \mathbb{R}^d$, which we call the canonical extension.
    Then ${\rm P}={\rm Q}$ if and only if there exists  an open set $\mathcal{D}$ with $\mathrm{supp} ( \mathcal{U}) \subset \mathcal{D} \subset \Omega$ such that $g^{\rm P}_{\mathcal{U}}(\u)=g^{\rm Q}_{\mathcal{U}}(\u)$, for $\ell_d-$a.e. $\u\in \mathcal{D}$ (after extension). Moreover, if there exists  an open set  $\mathcal{D}' \subset \Omega$ such that $\ell_d \ll  \mathcal{U}$ in $\mathcal{D}'$, then $\operatorname{Var}_{\U\sim\mathcal{U}}(g^{\rm P}_{\mathcal{U}}(\U)-g^{\rm Q}_{\mathcal{U}}(\U))=0$ if and only if ${\rm P}={\rm Q}$.
    \end{proposition}

    \begin{corollary} \label{cor:strict:PD:one}
   Let $\mathcal{U} \in \mathcal{P}_{SG}(\Omega)$ and assume that there exists  an open set  $\mathcal{D}' \subset \Omega$ such that $\ell_d \ll  \mathcal{U}$ in $\mathcal{D}'$.
Assume also that $F$ in Theorem \ref{theorem:F:kernel} is non-constant. Then
        the function $K$ in Theorem \ref{theorem:F:kernel} is strictly positive definite on $\mathcal{P}_{SG}(\Omega)$.
       \end{corollary}
    
    \begin{remark}\label{Remmark:discrete}
    The previous result is stated for the case where $\mathcal{U}$ dominates Lebesgue measure on a ball; in particular $\mathcal{U}$ cannot be discrete.
    Nevertheless, even when $\mathcal{U}$ does not satisfies this assumption, we can still construct a strictly positive definite kernel as follows.
    Let $\mathcal{U},{\rm P}, {\rm Q} \in \mathcal{P}_{SG}(\Omega)$.
    First note that from \eqref{eq:optimality:conditions:un} and \eqref{eq:optimality:conditions:deux} we have 
    $$g^{\rm P}_{\mathcal{U}}(\y)=-\epsilon\log \int \exp\left(-\frac{1}{\varepsilon}\left(f^{\rm P}_{\mathcal{U}}(\x)-\frac{1}{2}\|\x-\y\|^2 \right) \right) d{\rm P}(\x)  $$
    which extends $g^{\rm P}_{\mathcal{U}}$ out of the support of $\mathcal{U}$, on an open ball $B$ of $\mathbb{R}^d$ containing $\mathrm{supp}(\mathcal{U})$. 
    Then we have $g^{\rm P}_{\mathcal{U}} = g^{\rm Q}_{\mathcal{U}}$ $\ell_d$-a.e. on $B$  implies $P = Q$.
    Thus, let $ F$ be as in Theorem~\ref{theorem:F:kernel} and assume that it is non-constant. Define the function $K: \mathcal{P}_{SG}(\Omega) \times \mathcal{P}_{SG}(\Omega) \mapsto \mathbb{R}$ as 
    \begin{equation} \label{eq:kernel2}
        K({\rm P},{\rm Q})=F(\|g^{\rm P}_{\mathcal{U}}-g^{\rm Q}_{\mathcal{U}}\|_{L^2(\ell_d,B)}),
    \end{equation}
    with $\|. \|_{L^2(\ell_d,B)}$ the square norm w.r.t. the measure $\ell_d$ on $B$. Then
    $K$ is strictly definite positive on $\mathcal{P}_{SG}(\Omega) \times \mathcal{P}_{SG}(\Omega)$.
    \end{remark}

       As an example for Remark \ref{Remmark:discrete},  set $\epsilon=1$, suppose $\mathbf{0}\in \Omega$ and consider the discrete measure $\mathcal{U}=\delta_{\mathbf{0}}$. In this case,
    $$ g^{{\rm P}}_{\mathcal{U}}({\bf y})=\frac{\|{\bf y}\|^2}{2}-\log(\operatorname{M}_{{\rm P}}({\bf y})),$$
    for ${\bf y}$ in a neighborhood of ${\bf 0}$, after extension,
    where $\operatorname{M}_{{\rm P}}({\bf y})=\int e^{\langle {\bf y},{\bf x} \rangle}d {\rm P}({\bf x})$ is the moment generating function.

     A kernel $K$ is said to be universal on $ \mathcal{P}(\Omega)$ as soon as the space generated by all possible  linear combinations $\mu \mapsto \sum_{i=1}^n \alpha_i K(\mu,\mu_i) $  has  good approximation properties, in the sense that it is dense in the set continuous functions on $\mathcal{P}(\Omega)$, endowed with the weak convergence of probabilities. We prove  that the squared exponential kernel built with the distance between the potentials $\| g^P_{\mathcal{U}}-g^Q_{\mathcal{U}}\|_{L^2(\mathcal{U})} $ is universal.
    \begin{proposition}[Universality of Sinkhorn based kernel] \label{proposition:universality}
    Assume that $\Omega$ is compact and that
    there exists  an open set  $\mathcal{D}' \subset \Omega$ such that $\ell_d \ll  \mathcal{U}$ in $\mathcal{D}'$
    Consider for every distribution $P,Q$ in $\mathcal{P}(\Omega)$, their potentials $g^P_\mathcal{U}$, $g^Q_\mathcal{U}$ as in \eqref{potP} and \eqref{potQ}. Then for any $\sigma>0$ the kernel defined by $$ K_{\sigma}(P,Q)=\exp(-\sigma \| g^P_\mathcal{U}-g^Q_\mathcal{U}\|^2_{L^2(\mathcal{U})} )$$ is universal.
    \end{proposition}

\subsection{Consistency property of the empirical Kernel}

    In practical situations, the distributions may not be known but only random samples may be at hand. Let ${\X}_1, \dots, {\X}_n$ and  ${\Y}_1, \dots, {\Y}_m$ be mutually independent sequences of random vectors with distributions ${\rm P}$ and ${\rm Q}$ respectively. Denote as ${\rm P}_n$ and ${\rm Q}_m$ their empirical measures: ${\rm P}_n = (1/n) \sum_{i=1}^n \delta_{\X_i }$ and ${\rm Q}_m = (1/m) \sum_{i=1}^m \delta_{\Y_i }$. Consider the optimal entropic transport potentials of the empirical distributions towards a common fixed measure $\mathcal{U}$ denoted by $(f^{{\rm P}_n},g^{{\rm P}_n})$ and $(f^{{\rm Q}_m},g^{{\rm Q}_m})$. Finally, define the empirical kernel by 
    $ 
    K({\rm P}_n,{\rm Q}_m)= F (\|g^{{\rm P}_n} -g^{{\rm Q}_m} \|_{L^2(\mathcal{U})})$.
    The following proposition proves its consistency.
    
    \begin{proposition}[Consistency of the empirical kernel] \label{proposition:consistency:empirical}
    Assume that $\Omega$ is compact and let ${\rm P}, {\rm Q} \in \mathcal{P}(\Omega)$.
    When $F$ is continuous, the empirical kernel $K({\rm P}_n,{\rm Q}_m)$ converges almost-surely when both $n,m \rightarrow  \infty$ to the true kernel $K({\rm P},{\rm Q})$. Moreover if we assume that  $F$  satisfies $|F(t)- F(s)| \leq A |t-s|^a$ for constants $0<A<\infty$ and $0 < a \leq 1$ and for $t \geq 0$, then we have the following bound, with a constant $C_d$,
    \begin{equation}
    \begin{aligned}
       &\mathbb{E} |  K({{\rm P}_n},{{\rm Q}_m}) - K({\rm P},{\rm Q}) | \leq\\
       &C_{d}\left(\left(\frac{1}{\sqrt{n}}+\frac{1}{\sqrt{m}}\right)\operatorname{diam}({\Omega})^{2^{d+1}} e^{ \frac{19}{2}\operatorname{diam}({\Omega})^4} \right)^a.
    \end{aligned}
    \end{equation}
    \end{proposition}

\subsection{Influence of $\mathcal{U}$ and invariance properties.}
 In this section we investigate the impact of the reference distribution $\mathcal{U}$. 
    Consider two distributions $\mathcal{U}$ and $\mathcal{U}'$ that will be used to build two different kernels. Consider the Sinkhorn OT towards respectively $\mathcal{U}$ and $\mathcal{U}'$ for both distribution ${\rm P}$ and ${\rm Q}$. We will write the corresponding entropic potentials as $(f_{\mathcal{U}}^P,g_{\mathcal{U}}^{\rm Q})$ and $(f_{\mathcal{U}'}^{\rm P},g_{\mathcal{U}'}^{\rm Q})$ (defined as in Section \ref{subsection:construction:positive:kernel}). We thus have the two kernels
    $ K_\mathcal{U}({\rm P},{\rm Q})=F( \| g^P_\mathcal{U}-g_\mathcal{U}^{\rm Q} \|_{L^2(\mathcal{U})}) $
    and
    $ K_{\mathcal{U}'}({\rm P},{\rm Q})=F( \| g^P_{\mathcal{U}'}-g_{\mathcal{U}'}^{\rm Q} \|_{L^2(\mathcal{U}')})$.
    One desirable property is translation invariance, which means that the kernel does not change whenever the reference distribution is changed by translation. This is shown next for our kernel construction.
    \begin{lemma} \label{lemma:invariance:translation}
    
    Let $\mathcal{U}, {\rm P}, {\rm Q} \in \mathcal{P}_{SG}(\Omega)$  and  $T_0 : \Omega \to \Omega$ be a translation ($T_0(\u) = \u + \u_0$ for a fixed $\u_0 \in \mathbb{R}^d$), then
    $ K_{T_0\sharp\mathcal{U}}({\rm P},{\rm Q})= K_{\mathcal{U}}({\rm P},{\rm Q})$.
    \end{lemma}

    The choice of the reference measure merits an important discussion in this work. Indeed, many possible probabilities with density could be used. For the sake of applications, uniformly distributed measures (on the square, on the ball, spherical uniform) are beneficial, as they are easily approximated on a discrete set. Also, the uniform distribution allows us to compare the Sinkhorn potentials by  factorising into independent lower  dimensional marginals. Note that an issue of using a unit-squared reference is the high influence of the coordinate system, which can be arbitrary in some applications. \\
    
    A benefit of using a spherical (invariant to linear isometries) reference distributions is rotation invariance, as shown next.
    \begin{lemma} \label{lemma:rigid:transformation}
    Let $\mathcal{U}, {\rm P}, {\rm Q} \in \mathcal{P}_{SG}(\Omega)$, with
     $\mathcal{U}$ spherical  and  $T_0: \Omega \to \Omega$ be a rigid transformation (i.e. $T_0(\x)=\boldsymbol{R} \x+\boldsymbol{t}$  with $\boldsymbol{R}^T=\boldsymbol{R}^{-1}$ and $\boldsymbol{t} \in \R^d$), then
    $ K_{T_0\sharp\mathcal{U}}({\rm P},{\rm Q})= K_{\mathcal{U}}({\rm P},{\rm Q})$.
    \end{lemma}

   Again for spherical distributions, the following result shows that a dilatation of factor $\delta$ of the reference measure is  equivalent to a change of order $\epsilon=1/\delta^2$ on the Sinkhorn problem. 
      \begin{proposition} \label{prop:dilatation}
      Let $\mathcal{U}, {\rm P}, {\rm Q} \in \mathcal{P}_{SG}(\Omega)$, with
     $\mathcal{U}$ spherical and  $T_{\delta}(\u)=\delta\,\u$, with $\delta>0$, then
    \begin{align*} & \operatorname{Var}_{\U_{\delta}\sim T_{\delta}\sharp\mathcal{U}}(g^{\rm P}_{T_{\delta}\sharp\mathcal{U}}(\U_{{\delta}})-g^{\rm Q}_{T_{\delta}\sharp\mathcal{U}}(\U_{{\delta}})) \\
    & =\delta^4\operatorname{Var}_{\U\sim \mathcal{U}}(g^{T_{\frac{1}{\delta}}\sharp\rm P}_{\mathcal{U}, \, \delta}(\U)-g^{T_{\frac{1}{\delta}}\sharp\rm Q}_{\mathcal{U}, \, \delta}(\U)),
    \end{align*}
    where 
    $g^{T_{\frac{1}{\delta}}\sharp\rm P}_{\mathcal{U}, \, \delta}$ and $g^{T_{\frac{1}{\delta}}\sharp\rm Q}_{\mathcal{U}, \, \delta}$
    solve the dual formulation \eqref{dual_entrop} of $S_{\epsilon}(T_{\frac{1}{\delta}}\sharp\rm P, \mathcal{U})$ and $S_{\epsilon}(T_{\frac{1}{\delta}}\sharp\rm Q, \mathcal{U})$, for $\epsilon=\frac{1}{\delta^2}$. Above, $g^{\rm P}_{T_{\delta}\sharp\mathcal{U}}$ and $g^{\rm Q}_{T_{\delta}\sharp\mathcal{U}}$ correspond to $\epsilon = 1$.
    \end{proposition}
    
    For generic changes of reference distribution, the following proposition quantifies the corresponding kernel changes.
    \begin{proposition} \label{proposition:control:change:reference:U}
    Assume that $\Omega$ is compact. Let $s \in \mathbb{N}$.
    There exists a constant $c(\Omega,d,\epsilon,s)$ such that for $\mathcal{U}, \mathcal{U}', {\rm P}, {\rm Q} \in \mathcal{P}(\Omega)$, 
    \begin{multline*}
        | K_\mathcal{U}({\rm P},{\rm Q})- K_{\mathcal{U}'}({\rm P},{\rm Q})| \leq 2 {\rm diam}(\Omega) \| \mathcal{U}-\mathcal{U}'\|_s \\
+    c(\Omega,d,\epsilon,s) \left( \| \mathcal{U}-\mathcal{U}'\|_s  \|{\rm P}-{\rm Q}\|_s\right)^{1/2}.
    \end{multline*}

    \end{proposition}


\section{Gaussian processes using Sinkhorn's potential kernel} \label{s:GP}

Let us recall that a GP $(Z(x))_{x \in E}$ indexed by a set $E$ is entirely characterised by its mean and covariance functions. Its covariance function is defined by $(x,y) \in E^2 \mapsto \mathrm{Cov}(Z(x),Z(y))$. In this section we consider the GP on distributions defined by the Sinkhorn's potential Kernel  $K$ with  $\mathrm{Cov}(Z(P),Z(Q))=K(P,Q)$ with $K$ as in Theorem \ref{theorem:F:kernel}.  We study its property in this section. 

\subsection{Continuity of the Gaussian process}

For any positive definite kernel, a GP is guaranteed to exist having this kernel as covariance function. Nevertheless, this GP is defined only as a collection of Gaussian variables, and not necessarily as a random continuous function. Being able to define a GP as a random continuous function is at the same time satisfying from a functional Bayesian point of view, and also technically useful to tackle advanced convergence results, see for instance \cite{bect2019supermartingale}. Next, we establish the existence of a continuous GP with our kernel construction, under mild regularity assumptions on the space of input probability measures.

For a set $S \subset \mathbb{R}^d$, we let $\partial S$ be its boundary and for $\bt \in \mathbb{R}^d$, we let $d(\bt,S)$ be the smallest distance between $\bt$ and an element of $S$.  
\begin{proposition} \label{proposition:continuity}
Let $\Omega$ be compact with non-empty interior. Let $F$ in \eqref{eq:kernel} satisfy $|F(t)- F(0)| \leq A |t|^a$ for constants $0<A<\infty$ and $0 < a \leq 1$ and for $t \geq 0$.
Let $b >0$ be fixed.
Let $\mathcal{P}_{\delta}$ be the set of distributions ${\rm P}$ on $\Omega$ that have a continuous density $p$ w.r.t. Lebesgue measure, such that $p$ is zero on $\{ \x \in \Omega , d(\x , \partial \Omega)  \leq b\}$.
Consider $\mathcal{P}_{\delta}$ as a metric space with the 1-Wasserstein distance $\mathcal{W}_1$.  
Then there exists a GP $Z$ on $\mathcal{P}_{\delta}$ with covariance function as in \eqref{eq:kernel} that is almost surely continuous on $\mathcal{P}_{\delta}$.
\end{proposition}
The  proof of Proposition~\ref{proposition:continuity} is based on a control of the covering numbers of the canonical distance defined through the covariance function in \eqref{eq:kernel}.  A multi-dimensional integration by part allows us to upper bound this quantity by the covering numbers of  $\mathcal{C}^{\left\lceil \frac{d}{a}+1 \right\rceil}(\Omega)$, which is enough for the continuity of the process (see  \cite[Theorem 2.7.1]{vandervaart2013weak} and \cite[Theorem 1.1]{adler1990introduction}).

\subsection{Estimation of the parameters and prediction}
\label{section:GP:estimation:prediction}
{\bf Parametrization of the kernel.} 
The kernel is
$$K_{\btheta,\u}({\rm P},{\rm Q})=F_{\btheta} (\|g_{\u}^{\rm P}-g_{\u}^{\rm Q}\|_{L^2(\mathcal{U})} ),$$
where $F_{\btheta}$ is the function $F$ in Theorem \ref{theorem:F:kernel}, depending on the vector of covariance parameters $\btheta$. For instance for the square exponential covariance function, $\btheta$ consists in a scalar variance and length scale.
Furthermore, the Hilbertian embedding yielding $g_{\u}^{\rm P}$ and $g_{\u}^{\rm Q}$ depends on the choice of the  reference measure $\mathcal{U}$ (see Section \ref{subsection:construction:positive:kernel}). This choice is indexed by a vector $\u$. For instance, in our numerical experiments, $\mathcal{U}$ will be a discrete measure and $\u$ gathers the support points and weights. 
The presentation of (standard) likelihood methods for selecting ${\btheta,\u}$ in regression and classification, together with a discussion on microergodicity, are given in the Appendix, for the sake of brevity.

{\bf Prediction.}
The GP framework enables to predict the outputs corresponding to new input probability measures, by using conditional distributions given observed outputs.
This is reviewed in the Appendix for regression and classification.


\section{Implementation and experiments} \label{s:implementation}
\begin{figure*}
    \centering
    \includegraphics[scale=0.20]{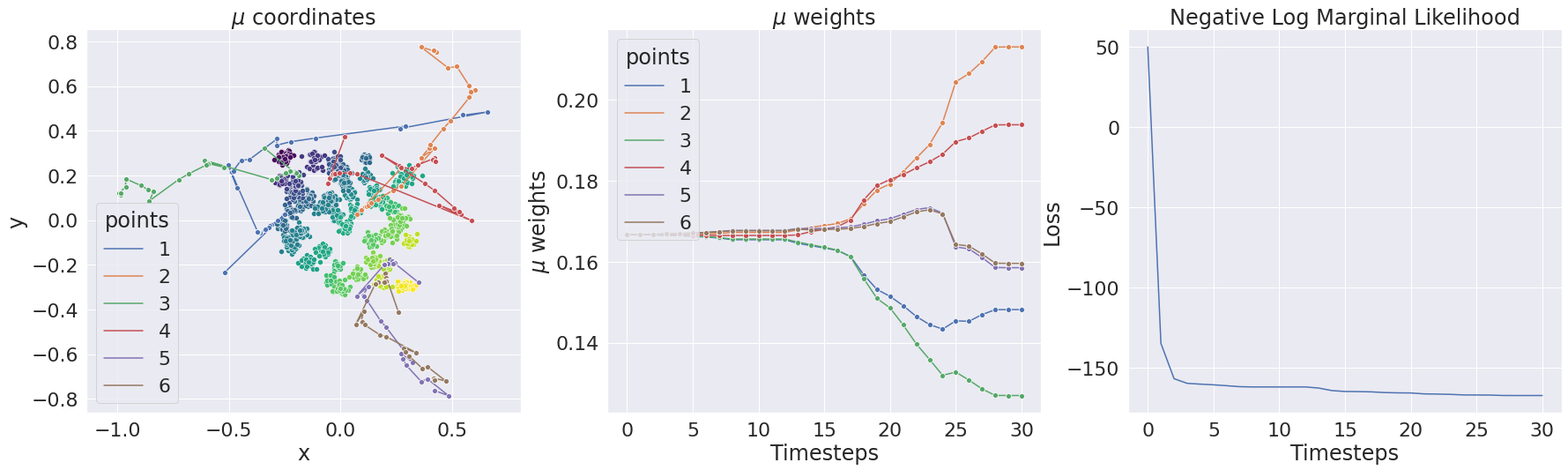}
    \caption{Toy example. \textbf{Left}: 50 point clouds of the train set, with color scale depending on random field $Z$. Trajectory of the points ${\bf x}_i$ of $\u$ depicted in different colors. \textbf{Center}: evolution of the weights $\bf w$ of $\u$ during training. \textbf{Right}: evolution of Negative Log Marginal Likelihood during training.}
    \label{fig:toy_example}
\end{figure*}

{\bf {Parametrization of the reference measure $\mathcal{U}$}.}
We chose a suitable machine representation for $\mathcal{U}$ (see Section \ref{section:GP:estimation:prediction}) as a weighted sum of Diracs:
$$\mathcal{U}=\sum_{i=1}^{q}w_i\delta(\x_i)\text{ with }
\sum_{i=1}^q w_i =1,
w_i \geq 0,
{\bf x}_i\in\R^d.$$
In this form $\mathcal{U}$ is not absolutely continuous w.r.t. Lebesgue measure, however the kernel remains strictly positive definite thanks to Remark \ref{Remmark:discrete}.  
The parameters $\u$ for $\mathcal{U}$ gather $w_1,\ldots,w_q,\x_1,\ldots,\x_q$. The procedure for the estimation of $\u,\btheta$ is sketched in Algorithm~\ref{alg:sinkhornu}.  

{\bf Gradient computations.}
We will use the
L-BFGS method for optimization \cite{liu1989limited}. This requires the gradients of the likelihood function in regression and classification w.r.t. $\btheta$ and $\u$. The derivatives of relevant quantities w.r.t. $\btheta$ can be found in the literature, see for instance \cite{rasmussen2006gaussian}. A specificity of $\u$ is that for some measures ${\rm P}, {\rm Q}$, we need to differentiate $\|g_{\u}^{\rm P}-g_{\u}^{\rm Q}\|_{L^2(\mathcal{U})}$ w.r.t. $\u$, that is we need to differentiate regularized OT plans. This is possible either by back-propagating through unrolled Sinkhorn iterations~\cite{genevay2018learning}, or by using implicit differentiation~\cite{eisenberger2022unified}. In practice we noticed that, while being slower, unrolling of Sinkhorn iterates was actually more stable numerically.  

{\bf {Software framework used.}}
For automatic support of autodifferentiation, we use Jax framework~\cite{jax2018github} with libraries GPJax~\cite{Pinder2022} to implement GP regression, OTT-Jax~\cite{cuturi2022optimal} for differentiable Sinkhorn algorithm, and Jaxopt~\cite{jaxopt_implicit_diff} for optimization with L-BFGS. The computation of inverse covariance matrices is done efficiently using Cholesky decomposition~\cite{press2007numerical}, which allows efficient computation of \textit{matrix inverse-vector products} without materializing the inverse in memory. The computations are performed in \textit{float32} arithmetic and take advantage of GPU for matrix operations, that are the bottleneck of the algorithm.  

{\bf Other numerical aspects.}
For $\u$, the point coordinates are parameterized as ${\bf x}= S \tanh{(\tilde {\bf x})}$ with $S\in\R$ to ensure they remain bounded, the weights are parametrized as $\bf w=\text{softmax}(\tilde {\bf w})$ to ensure they represent a valid probability distribution. The dual variables $g^{\rm P}_{\u}$ computed at each time step during the optimization of $\u$ are cached to speed-up Sinkhorn iterations: this strategy is reasonable since when $\u$ and $\u'$ are close then the dual variables $g^{\rm P}_{\u}$ and $g^{\rm P}_{\u'}$ are close too. \\

{\bf Computational cost of $\u$-Sinkhorn kernels.}
We denote by $|\u|$ the size of the support of $\u$ (written $q$ above). For another point cloud of size $n$, according to \cite{altschuler2017near,dvurechensky2018computational} the time complexity of Sinkhorn algorithm is $\O(\frac{n|\u|\log{(n|\u|)}}{\epsilon^2})$ to reach precision $\epsilon$, while the complexity of Maximum Mean Discrepancy (MMD) kernel is $\O(n^2)$. It follows that for a reference measure with  $|\u| \ll n$ with small support the runtime cost of Sinkhorn $\u$-kernel becomes competitive. Runtime against MMD is reported in Table~\ref{tab:runtimecost} (in appendix), with a speed-up of up to 100 for our method.  
  
Once $\u$ is chosen, the embeddings $g^{\rm P}_{\u}$ can be pre-computed once for all for each point cloud ${\rm P}_1,\ldots,{\rm P}_n$ and used as a low dimension embedding of $\mathcal{P}(\Omega)$ into $\R^{|\u|}$. The distribution support $|\u|$ needs to be big enough to capture the similarities between the ${\rm P}_i$s up to the precision required by the task, but does not need to be bigger (see Section~\ref{subsec:toy}).  



\begin{table}[]
    \centering
    \small
    \begin{tabular}{c|ccc|c}
         Task & $|\u|$ & $m$ & Ours & \cite{bachoc2020gaussian} \\ 
         \hline
         Toy example & 6 & 30 & {\bf 0.997} & 0.81 
    \end{tabular}
    \caption{Explained Variance Score (EVS) on the test set for regression tasks, with train set of size $n=50$ in dimension $d$=2. $|\u|$: dimension of the embedding. $m$: cloud size.}
    \label{tab:regressionscores}
\end{table}

\begin{table}[]
    \centering
    \footnotesize
    \begin{tabular}{c|cc|cc}
         Task & $|\u|$ & Ours & RBF\\
         \hline
         ``4'' vs ``6'' & 4 & $94.2\pm 1.2$ & \xmark \\
         ``4'' vs ``6'' & 5 & $95.5\pm 1.0$ & \xmark \\
         ``4'' vs ``6'' & 6 & $95.0\pm 0.6$ & $98.8\pm 0.2$\\
         ``shirt'' vs ``sandals'' & 12 & $99.5\pm 0.2$ & $99.7\pm 0.2$\\
         ``sneakers'' vs ``sandals'' & 12 & $88.6\pm 1.8$ & $91.9\pm 1.2$
    \end{tabular}
    \caption{Test Accuracy for \textbf{classification} tasks, with train set of size $n=200$ in dimension $d$=2 with clouds of size $m=24\times 24=576$. $|\u|$: dimension of the embedding. We compare against Radial Basis Function (RBF). Average over $25$ runs.}
    \label{tab:classificationscores}
\end{table}

\begin{figure*}
    \centering
    \includegraphics[scale=0.20]{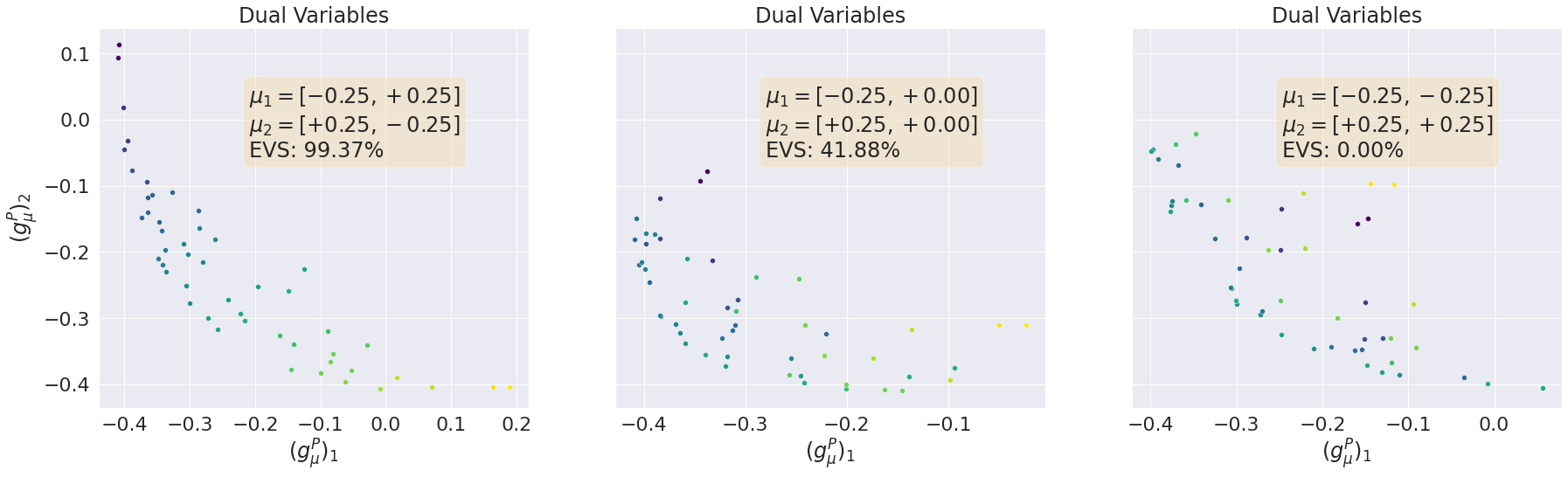}
    \caption{Role of $\u$ in quality of embeddings when $|\u|=2$ for example of Section~\ref{subsec:toy}. Each dot is the 2D embedding of a Gaussian where the color depends on the random field $Z$. \textbf{Left}: optimal choice for $\u$ that ensures the task can be solved. \textbf{Center}: sub-optimal choice for $\u$. \textbf{Right}: bad choice of $\u$ that prevents learning.}
    \label{fig:murole}
\end{figure*}
\vspace{-0.2cm}
\vspace{-0.2cm}
\label{s:expe}

\label{subsec:toy}

\subsection{Regression on toy example of~\cite{bachoc2020gaussian}} \vspace{-0.2cm}
In this section we re-use the example introduced in Section 5.3 of \cite{bachoc2020gaussian}. We simulate 100 random two-dimensional isotropic Gaussian distributions. The means are sampled uniformly from $[-0.3,0.3]^2$, and the variance uniformly from $[0.01^2,0.02^2]$. The value of the random field induced by a Gaussian of means $(m_1,m_2)$ and variance $\sigma^2$ is $Z=\frac{(m_1+0.5-(m_2+0.5)^2)}{1+\sigma}$. Gaussians are approximated by point clouds of size 30 sampled from the distribution. The data set is splitted into train (50 clouds) and test (50 clouds). The measure $\u$ consists of $6$ points on the ball of radius $0.5$. Their position $\x_i$ and weight $w_i$ are trained for $30$ iterations jointly with kernel parameters. Results are highlighted in Figure~\ref{fig:toy_example} and Table~\ref{tab:regressionscores}. The role of $\u$ is investigated in Figure~\ref{fig:murole} with $|\u|=2$: the position of $\x_i$'s make the embedding more or less suitable for the downstream task, as illustrated by Explained Variance Score (EVS) score.  

\begin{figure}
    \centering
    \includegraphics[scale=0.16]{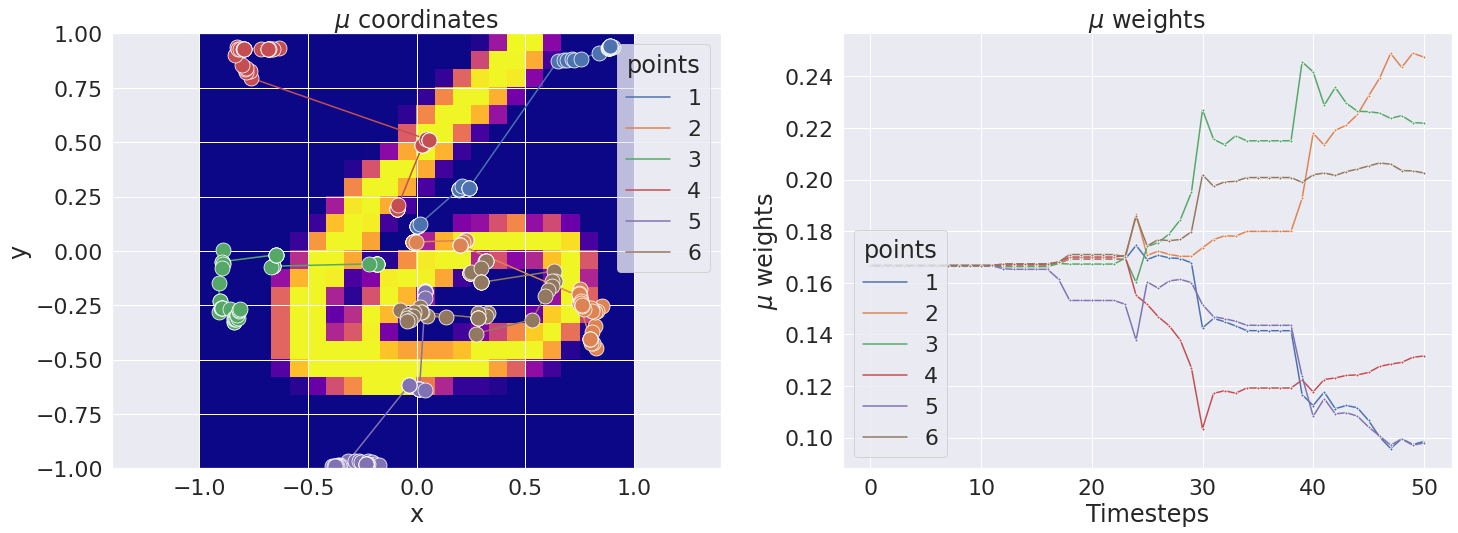}
    \caption{Optimization of $\u$ on Mnist ``4'' versus ``6'' task with $|\u|=6$. An image from the train set is displayed on the background to better grasp the scale of $\u$. Weights $w_i$ and positions $x_i$'s are moved to maximize the log marginal likelihood.}
    \label{fig:mnistexample}
\end{figure}

\subsection{Binary classification on Mnist and Fashion-Mnist}

We perform binary classification on Mnist by learning to separate digits ``4'' and ``6''. The dataset consists of $200$ train images, and $1000$ test images. Each $28\times 28$ images is centered crop to $24\times 24$ to generate a cloud of size $576$ matching pixel coordinates. The normalized pixel intensity is used as a weight in OT. The likelihood is modeled with Bernoulli distributions (not Gaussian, see the Appendix on GP classification), and the log marginal likelihood is maximized using maximum a posteriori (MAP) estimates. We tested different sizes for $|\u|\in [4, 5, 6]$. The training is depicted in Figure~\ref{fig:mnistexample}. The experiment is repeated $10$ time with random splits. It shows that Mnist images can be embedded in a space of small dimension that preserves most information about labels, achieving a compression rate of $R=\frac{|\u|}{584}\in [0.006, 0.013]$ tailored for the learning task.  
  
On Fashion-Mnist the ``shirt'' vs ``sandals'' task is surprisingly easy, whereas ``sneakers'' vs ``sandals'' is harder. Embeddings of sizes $8$ and $12$ were sufficient.

\begin{algorithm}
    \caption{Learn Kernel parameters.}
    \begin{algorithmic}[1]
    \STATE \textbf{input} $(P_i,y_i)_{1\leq i\leq N}$: dataset of distributions.
    \STATE \textbf{input} $\btheta_0=(\u_0,\sigma_0,l_0)$: initial parameters.
    \REPEAT
        \FORALL{$P_i$}
        \STATE Solve regularized OT problem between $P_i,\u_t$.
        \STATE Compute Sinkhorn dual potential $g^{P_i}_{\u_t}$.
        \ENDFOR
        \STATE Build Kernel $K_{ij}\defeq l\exp{-\frac{\|g^{P_i}_{\u_t}-g^{P_j}_{\u_t}\|}{2\sigma^2}}$.
        \STATE Compute log marginal likelihood $\mathcal{L}(\u_t,\sigma,l,K,y)$.
        \STATE Compute gradients $\nabla_{(\u,\sigma,l)}\mathcal{L}$ with Auto-Diff.
        \STATE Perform one step of L-BFGS on $(\u_t,\sigma,l)$.
    \UNTIL{convergence of $(\u_t,\sigma,l)$.}
    \STATE \textbf{Return} optimal parameters $(u_{*},\sigma_{*},l_{*})$.
    \end{algorithmic}
    \label{alg:sinkhornu}
\end{algorithm}

\subsection{Texture classification with C-SVM}
We follow the experimental procedure of~\cite{kolouri2016sliced} on the University of Illinois Urbana Champaign (UIUC) texture dataset~\cite{lazebnik2005sparse}. We transform the images into two dimensional probability distributions by computing the gray-level
co-occurence matrices (GLCM)~\cite{haralick1973textural}. The C-SVM optimization problem is a quadratic programming problem. When the kernel is Positive Definite the associated quadratic form is convex. This guarantees that the algorithm will converge to a global minimum. Our kernel matches the performances of~\cite{kolouri2016sliced} on the same experimental protocol, in Table~\ref{tab:uiucresult} (see Appendix).  



\vspace{-0.2cm}
\section{Conclusion}
In this paper we proposed a new positive definite kernel tailored for distributions. It is universal, allows to embed distributions in a space of smaller dimension controlled by the size of $\u$, and is consistent so it scales with the number of points available to approximate the distribution. Empirically, we showed that the reference measure $\u$ was of crucial importance and could be optimized directly with maximum likelihood. Our numerical experiments also highlight that our kernel yields a similar accuracy as other methods, while providing an important computational speed-up.



\onecolumn



\section{Proofs and some additional results for Section~\ref{sec:kerneldef}}

    
    \begin{proof}[Proof of Theorem \ref{theorem:F:kernel}]
    The result follows from Proposition 4 and Remark 5 in \cite{bachoc2020gaussian}. 
    \end{proof}
  
    \begin{proof}[Proof of Proposition~\ref{Lemma:cuadratic}]
        The proof can be obtained \emph{mutatis mutandis} from that of the empirical case  \cite[Theorem 4.5]{Barrio2022AnIC}.
    \end{proof}


    \begin{proof}[Proof of Proposition \ref{Lemma:uniqueness}]
    For ease of notation we suppose that $\epsilon=1$. We prove both equivalences at the same time. In any of the assertions, ${\rm P}={\rm Q}$ implies the equality of the potentials in $\R^ d$--defined via the canonical extension given by the optimality condition \eqref{eq:optimality:conditions:un}. On the other hand, let us suppose that $g^{\rm P}(\u)=g^{\rm Q}(\u)$, for $\ell_d$-a.e. $\u\in \mathcal{D}$, for some open set $\mathcal{D}\subset\R^ d$.
    Then $g^{\rm P}=g^{\rm Q}$, $\mathcal{U}$-a.e. by continuity when $\mathrm{supp}(\mathcal{U}) \subset \mathcal{D}$.
   In consequence, the other potentials, obtained by the relations
    \begin{align*}
        &f^{\rm P}(\x)=-\log\left(\int e^{{g^{\rm P}(\y)- \frac{1}{2}\|\x-\y\|^2}}d\mathcal{U}(\y) \right),\ \
        f^{\rm Q}(\x)=-\log\left(\int e^{{g^{\rm Q}(\y)- \frac{1}{2}\|\x-\y\|^2}}d\mathcal{U}(\y) \right),
    \end{align*}
    are also equal. Moreover, since $g^{\rm P}=g^{\rm Q}$ $\mathcal{U}$-a.s. then $e^{g^{\rm P}}=e^{g^{\rm Q}}$ too, and, using the optimally conditions, we have 
    \begin{align*}
        &\int e^{f^{\rm P}(\x)- \frac{1}{2}\|\x-\y\|^2}dP(\x)=e^{g^{\rm P}(\y)}=\int e^{f^{\rm Q}(\x)- \frac{1}{2}\|\x-\y\|^2}dQ(\x), \ \ \text{for $\mathcal{U}$-a.e.  $\y\in \Omega$}.
    \end{align*}
    Moreover, extending $e^{g^{\rm P}(\y)}$ as in Remark~\ref{Remmark:discrete}, we obtain 
     \begin{align*}
        &\int e^{f^{\rm P}(\x)- \frac{1}{2}\|\x-\y\|^2}dP(\x)=e^{g^{\rm P}(\y)}=\int e^{f^{\rm Q}(\x)- \frac{1}{2}\|\x-\y\|^2}dQ(\x), \ \ \text{for all  $\y\in \mathcal{D}$},
    \end{align*}
    so, due to the equality $f^{\rm P}=f^{\rm Q}$ in $\Omega$, we have the equality 
    \begin{align*}
        &\int e^{\langle \x,\y \rangle}e^{f^{\rm P}(\x)- \frac{1}{2}\|\x\|^2}dP(\x)=\int e^{\langle \x,\y \rangle}e^{f^{\rm P}(\x)- \frac{1}{2}\|\x\|^2}dQ(\x), \ \ \text{for all  $\y\in \mathcal{D}$}.
    \end{align*}
    By hypothesis, and without loosing generality, there exists a ball centered in $0$, such that $\mathbb{B}_{\epsilon}(0)\subset \mathcal{D}$.  A \emph{fortiori} $$\int e^{\langle \x,\y \rangle}e^{f^{\rm P}(\x)- \frac{1}{2}\|\x\|^2}dP(\x)=\int e^{\langle \x,\y \rangle}e^{f^{\rm P}(\x)- \frac{1}{2}\|\x\|^2}dQ(\x),$$ for all $\y\in \mathbb{B}_{\epsilon}(0)$. In particular, its  evaluation in  $\y=0$ yields the inequality $\int e^{f^{\rm P}(\x)- \frac{1}{2}\|\x\|^2}dP(\x)=\int e^{f^{\rm P}(\x)- \frac{1}{2}\|\x\|^2}dQ(\x)>0$. The uniqueness of the moment generating function, (see eg. \cite[Theorem 22.]{Bill86}) proves that the probabilities $\frac{e^{f^{\rm P}(\x)- \frac{1}{2}\|\x\|^2}}{\int e^{f^{\rm P}(\x)- \frac{1}{2}\|\x\|^2}dP(\x)}dP(\x)$ and $\frac{e^{f^{\rm P}(\x)- \frac{1}{2}\|\x\|^2}}{\int e^{f^{\rm P}(\x)- \frac{1}{2}\|\x\|^2}dQ(\x)}dQ(\x)$ are equal, so that ${\rm P}={\rm Q}$ too.
    \end{proof}
    
    \begin{proof}[Proof of Corollary \ref{cor:strict:PD:one}]
    Consider two-by-two distinct measures $P_1,\ldots,P_n$. Then from Proposition \ref{Lemma:uniqueness}, the functions $g^{P_1}_{\mathcal{U}},\ldots,g^{P_n}_{\mathcal{U}}$ are two-by-two distinct in $L^2(\mathcal{U})$. Then the matrix 
    \[
    \left[ F(\|g^{\rm P_i}_{\mathcal{U}}-g^{\rm P_j}_{\mathcal{U}}\|_{L^2(\mathcal{U})}) \right]_{1 \leq i,j \leq n}
    \]
    is strictly positive definite from Proposition 4 in \cite{bachoc2020gaussian}.
    \end{proof}
    
    \begin{proof}[Proof of Remark \ref{Remmark:discrete}] 
 The fact that $g^{\rm P}_{\mathcal{U}} = g^{\rm Q}_{\mathcal{U}}$ $\ell_d$-a.e. on $B$ implies $P = Q$   holds from Proposition \ref{Lemma:uniqueness}. Then, strict positive definiteness is shown as in the proof of Corollary \ref{cor:strict:PD:one}.
    \end{proof}
    
    

  
    \begin{proof}[Proof of Proposition \ref{proposition:universality}]
    First note that  $\mathcal{P}(\Omega)$ is a compact metric space endowed with the Wasserstein distance. Consider the map $\Phi$ from   $ \mathcal{P}(\Omega)$   to the separable Hilbert space $ L^2(\mathcal{U})$ such that for any $P \in \mathcal{P}(\Omega)$, $\Phi(P)=g^P_\mathcal{U}$. This map is continuous w.r.t. Wasserstein distance $\mathcal{W}_1$ from Proposition~\ref{Lemma:cuadratic} and the comment after it. Moreover  Proposition~\ref{Lemma:uniqueness} implies that $\Phi$ is injective. Hence using Theorem 2.2 in \cite{NIPS2010_4e0cb6fb}, we obtain the universality of the kernel.
    \end{proof}


    
    \begin{proof}[Proof of Proposition \ref{proposition:consistency:empirical}]
    First, using Proposition~\ref{Lemma:cuadratic}, we obtain that for some constant $C_{d,\Omega}$, we obtain the following bounds
    \begin{align}
        \begin{split}
             \label{comparisonPnP1}
     \|g^{{\rm P}_n}-g^{{\rm P}}\|_{L^2(\mathcal{U})}\leq  C_{d,\Omega}\|{\rm P}-{\rm P}_n\|_{s}, \ \text{and} \ \|g^{{\rm Q}_m}-g^{{\rm Q}}\|_{L^2(\mathcal{U})}\leq C_{d,\Omega}\|{\rm Q}-{\rm Q}_m\|_{s},
        \end{split}
    \end{align}
    where, in this case, $s= \lceil{\frac{2}{d}}\rceil+1$.   Moreover, the triangle inequality 
    $$\left|\|g^{{\rm P}_n}-g^{{\rm Q}_m}\|_{L^2(\mathcal{U})}- \|g^{{\rm P}}-g^{{\rm Q}}\|_{L^2(\mathcal{U})}\right|\leq   \|g^{{\rm P}_n}-g^{{\rm P}}\|_{L^2(\mathcal{U})}+ \|g^{{\rm Q}}-g^{{\rm Q}_m}\|_{L^2(\mathcal{U})}  $$
    and \eqref{comparisonPnP1} yield the upper bound
    \begin{align}
        \begin{split}
             \label{comparisonPnP}
     \left|\|g^{{\rm P}_n}-g^{{\rm Q}_m}\|_{L^2(\mathcal{U})}- \|g^{{\rm P}}-g^{{\rm Q}}\|_{L^2(\mathcal{U})}\right|\leq  C_{d,\Omega}\left(\|{\rm P}-{\rm P}_n\|_{s}+\|{\rm Q}-{\rm Q}_m\|_{s}\right).
        \end{split}
    \end{align}
    Hence when $n,m \rightarrow + \infty$, consistency of the empirical distributions and continuity of the function $F$ lead to the consistency of the empirical kernel almost surely.\\
    To obtain the upper bound,  note that using the Assumption on $F$ we have
    \begin{align*}
        \mathbb{E} |  K({{\rm P}_n},{{\rm Q}_m}) - K({\rm P},{\rm Q}) | & = \mathbb{E} |  F (\|g^{{\rm P}_n} -g^{{\rm Q}_m} \|_{L^2(\mathcal{U})} )- F (\|g^{{\rm P}} -g^{{\rm Q}} \|_{L^2(\mathcal{U})}) | \\
        & \leq 
        A
        \mathbb{E} | \|g^{{\rm P}_n} -g^{{\rm Q}_m} \|_{L^2(\mathcal{U})} - \|g^{{\rm P}} -g^{{\rm Q}} \|_{L^2(\mathcal{U})} |^a 
    \end{align*}
    Since $a\in(0,1]$, Jensen's inequality allows us to say that
    \begin{equation} \label{interm}  \mathbb{E} |  K({{\rm P}_n},{{\rm Q}_m}) - K({\rm P},{\rm Q}) |\leq A
    \left(\mathbb{E}  |  \|g^{{\rm P}_n} -g^{{\rm Q}_m} \|_{L^2(\mathcal{U})} - \|g^{{\rm P}} -g^{{\rm Q}} \|_{L^2(\mathcal{U})} | \right)^a. \end{equation}
    Therefore,  \eqref{comparisonPnP} and \eqref{interm} enable to obtain that 
    \begin{align}\label{aoutside}
    \mathbb{E} |  K({{\rm P}_n},{{\rm Q}_m}) - K({\rm P},{\rm Q}) |\leq 
    A
    \left( \mathbb{E}|  \|{\rm P}-{\rm P}_n\|_{s}+\|{\rm Q}-{\rm Q}_m\|_{s}|\right)^a.
    \end{align}
    The rest of the proof follows by classical empirical processes arguments.
    \end{proof}

    \begin{proof}[Proof of Lemma \ref{lemma:invariance:translation}]
    Let $\U\sim \mathcal{U}\in \mathcal{P}_{SG}(\Omega)$ and $\X\sim{\rm P}\in \mathcal{P}_{SG}(\Omega)$. The potentials are (up to additive constants) characterized by the optimally conditions 
    \begin{align*}
    \mathbb{E}\left( e^{ f^{\rm P}_{\mathcal{U}}(\X) + g^{\rm P}_{\mathcal{U}}(\u) - \frac{1}{2}\|\X-\mathbf{u}\|^2} \right) & =1 \quad \mathcal{U}-a.s.\\
    \mathbb{E}\left(  e^{ f^{\rm P}_{\mathcal{U}}(\x) + g^{\rm P}_{\mathcal{U}}(\U) - \frac{1}{2}\|\x-\U\|^2}  \right)&=1 \quad {\rm P}-a.s.
    \end{align*}
    Let $T_0$ be a translation---defined as $\u\mapsto\u+\u_0$---and  $\U_0=T_0(\U)\sim \mathcal{U}_0$, then we claim that 
    $$ 
    f^{\rm P}_{\mathcal{U}_0}(\x)=f^{\rm P}_{\mathcal{U}}(\x)+\langle\u_0,\x\rangle+\frac{3}{4} \|\u_0\|^2\quad \text{and}\quad g^{\rm P}_{\mathcal{U}_0}(\u)=g^{\rm P}_{\mathcal{U}}(\u-\u_0)-\langle\u_0,\u\rangle
    +\frac{3}{4}\|\u_0\|^2$$
    is a pair of OT potentials for $\mathcal{U}_0$. The verification of the optimallity conditions is enough to prove the claim.  On the one hand, note that
    \begin{align*}
        \mathbb{E}\left( e^{ f^{\rm P}_{\mathcal{U}_0}(\X) + g^{\rm P}_{\mathcal{U}_0}(\u') - \frac{1}{2}\|\X-\mathbf{u}'\|^2} \right)= \mathbb{E}\left( e^{ f^{\rm P}_{\mathcal{U}}(\X) +\langle\u_0,\X\rangle+g^{\rm P}_{\mathcal{U}}(\u'-\u_0) -\langle\u_0,\u'\rangle+\frac{3}{2}\|\u_0\|^2- \frac{1}{2}\|\X-\mathbf{u}'\|^2} \right)
    \end{align*}
    and the (evident) change of variables $\u=\u'-\u_0$ yields
    \begin{align*}
        \mathbb{E}\left( e^{ f^{\rm P}_{\mathcal{U}_0}(\X) + g^{\rm P}_{\mathcal{U}_0}(\u+\u_0) - \frac{1}{2}\|\X-\mathbf{u}+\u_0\|^2} \right)&= \mathbb{E}\left( e^{ f^{\rm P}_{\mathcal{U}}(\X) +\langle\u_0,\X\rangle+g^{\rm P}_{\mathcal{U}}(\u) -\langle\u_0,\u+\u_0\rangle+\frac{3}{2}\|\u_0\|^2- \frac{1}{2}\|\X-(\mathbf{u}+\u_0)\|^2} \right)\\
        &= \mathbb{E}\left( e^{ f^{\rm P}_{\mathcal{U}}(\X) +g^{\rm P}_{\mathcal{U}}(\u) - \frac{1}{2}\|\X-\mathbf{u}\|^2} \right)\\
        &=1 \quad {\mathcal{U}}-a.s.
    \end{align*}
    Therefore we obtain the first optimally condition 
    \begin{align*}
        \mathbb{E}\left( e^{ f^{\rm P}_{\mathcal{U}_0}(\X) + g^{\rm P}_{\mathcal{U}_0}(\u') - \frac{1}{2}\|\X-\mathbf{u}'\|^2} \right)= 1 \quad {\mathcal{U}_0}-a.s.
    \end{align*}
    On the other hand,  note that
    \begin{align*}
        \mathbb{E}\left( e^{ f^{\rm P}_{\mathcal{U}_0}(\x) + g^{\rm P}_{\mathcal{U}_0}(\U_0) - \frac{1}{2}\|\x-\U_0\|^2} \right)&= \mathbb{E}\left( e^{ f^{\rm P}_{\mathcal{U}}(\x) +\langle\u_0,\x\rangle+g^{\rm P}_{\mathcal{U}}(\U_0-\u_0) -\langle\u_0,\U_0\rangle+\frac{3}{2}\|\u_0\|^2- \frac{1}{2}\|\x-\U_0\|^2} \right)\\
         &= \mathbb{E}\left( e^{ f^{\rm P}_{\mathcal{U}}(\x) +\langle\u_0,\x\rangle+g^{\rm P}_{\mathcal{U}}(\U) -\langle\u_0,\U+\u_0\rangle+\frac{3}{2}\|\u_0\|^2- \frac{1}{2}\|\x-\U+\u_0\|^2} \right)\\
          &= \mathbb{E}\left( e^{ f^{\rm P}_{\mathcal{U}}(\x) +g^{\rm P}_{\mathcal{U}}(\U)- \frac{1}{2}\|\x-\U\|^2} \right)\\
        &=1 \quad {\rm P}-a.s.
    \end{align*}
    which implies the second optimally condition. 
    \end{proof}
    \begin{proof}
    [Proof of Lemma~\ref{lemma:rigid:transformation}]
    Via Lemma~\ref{lemma:invariance:translation}, we only need to prove the invariance w.r.t.  $T_0(\u)=\boldsymbol{R} \u$ with $\boldsymbol{R} $ a linear isommetry. By definition of spherical measure, $ T_0(\U)\sim \mathcal{U}$, for any $\U\sim \mathcal{U}$, so the solutions of \eqref{dual_entrop} are the same. 
    \end{proof}
  
    \begin{proof}[Proof of Proposition \ref{prop:dilatation}]
    Set $\mathcal{U}_{\delta}=T_{\delta}\sharp\mathcal{U}$ and a pair $(g^{\rm P}_{T_{\delta}\sharp\mathcal{U}},f^{\rm P}_{T_{\delta}\sharp\mathcal{U}})$ solving the  dual formulation \eqref{dual_entrop} of $S_{1}(\rm P,T_{{\delta}}\sharp\rm \mathcal{U})$. The optimality conditions yield
    \begin{align*}
    \mathbb{E}\left( e^{ f^{\rm P}_{T_{\delta}\sharp\mathcal{U}}(\X) + g^{\rm P}_{T_{\delta}\sharp\mathcal{U}}(\u_{\delta}) - \frac{1}{2}\|\X-\u_{\delta}\|^2} \right) & =1 \quad \mathcal{U}_{\delta}-a.s.
    \end{align*}
    where we can do a change of variables $\u=\frac{1}{\delta}\u_{\delta}$ to have
    \begin{align*}
    1&\stackrel{\mathcal{U}-a.s.}{=}\mathbb{E}\left( e^{ f^{\rm P}_{T_{\delta}\sharp\mathcal{U}}(\X) + g^{\rm P}_{T_{\delta}\sharp\mathcal{U}}({\delta}\u) - \frac{1}{2}\|\X-\delta\u\|^2} \right)\stackrel{\mathcal{U}-a.s.}{=}\mathbb{E}\left( e^{ f^{\rm P}_{T_{\delta}\sharp\mathcal{U}}(\X) + g^{\rm P}_{T_{\delta}\sharp\mathcal{U}}(\delta\u) - \frac{\delta^2}{2}\|\frac{1}{\delta}\,\X-\u\|^2} \right).
    \end{align*}
    Set $\X_{\frac{1}{\delta}}=T_{\frac{1}{\delta}}(\X)={\frac{1}{\delta}}\X$, then 
    \begin{align*}
    1&\stackrel{\mathcal{U}-a.s.}{=}\mathbb{E}\left( e^{ {\delta^2}\left(\frac{1}{\delta^2}\,f^{\rm P}_{T_{\delta}\sharp\mathcal{U}}({\delta}\X_{\frac{1}{\delta}}) + \frac{1}{\delta^2}\,g^{\rm P}_{T_{\delta}\sharp\mathcal{U}}({\delta}\u) - \frac{1}{2}\|\X_{\frac{1}{\delta}}-\u\|^2\right)} \right).
    \end{align*}
    The same argument also shows (with the obvious notation) that
    \begin{align*}
    1&\stackrel{\rm P_{\frac{1}{\delta}}-a.s.}{=}\mathbb{E}\left( e^{ {\delta^2}\left(\frac{1}{\delta^2}\,f^{\rm P}_{T_{\delta}\sharp\mathcal{U}}({\delta}\x_{\frac{1}{\delta}}) + \frac{1}{\delta^2}\,g^{\rm P}_{T_{\delta}\sharp\mathcal{U}}({\delta}\U) - \frac{1}{2}\|\x_{\frac{1}{\delta}}-\U\|^2\right)} \right),
    \end{align*}
    which means that the pair
    $ \left(\frac{1}{\delta^2}\,f^{\rm P}_{T_{\delta}\sharp\mathcal{U}}({\delta}\,\cdot\, ), \   \frac{1}{\delta^2}\,g^{\rm P}_{T_{\delta}\sharp\mathcal{U}}(\delta \, \cdot\,)\right)=\left(f^{T_{\frac{1}{\delta}}\sharp\rm P}_{\mathcal{U}, \, \delta}, \   \frac{1}{\delta^2}\,g^{T_{\frac{1}{\delta}}\sharp\rm P}_{\mathcal{U}, \, \delta}\right)$
    solves  the dual formulation \eqref{dual_entrop} of $S_{\epsilon}(T_{\frac{1}{\delta}}\sharp\rm P, \mathcal{U})$, for $\epsilon=\frac{1}{\delta^2}$. The same, verbatim, can be done for $\rm Q$. Finally, we note that 
    $$ \mathbb{E}\left(g^{\rm P}_{T_{\delta}\sharp\mathcal{U}}(\U_{{\delta}})-g^{\rm Q}_{T_{\delta}\sharp\mathcal{U}}(\U_{{\delta}})\right)^2=\mathbb{E}\left(g^{\rm P}_{T_{\delta}\sharp\mathcal{U}}(\delta\,\U)-g^{\rm Q}_{T_{\delta}\sharp\mathcal{U}}(\delta\,\U)\right)^2=\delta ^4\mathbb{E}\left(g^{T_{\frac{1}{\delta}}\sharp\rm P}_{\mathcal{U}, \, \delta}(\U)-g^{T_{\frac{1}{\delta}}\sharp\rm Q}_{\mathcal{U}, \, \delta}(\U)\right)^2,$$
    and 
    $$ 0=\mathbb{E}\left(g^{\rm P}_{T_{\delta}\sharp\mathcal{U}}(\U_{{\delta}})\right)=\delta ^2\mathbb{E}\left(g^{T_{\frac{1}{\delta}}\sharp\rm P}_{\mathcal{U}, \, \delta}(\U)\right), \ 0=\mathbb{E}\left(g^{\rm Q}_{T_{\delta}\sharp\mathcal{U}}(\U_{{\delta}})\right)=\delta ^2\mathbb{E}\left(g^{T_{\frac{1}{\delta}}\sharp\rm Q}_{\mathcal{U}, \, \delta}(\U)\right),$$
    which allows to conclude.
    \end{proof}

    \begin{proof}[Proof of Proposition \ref{proposition:control:change:reference:U}]
    For ease of notation we suppose that $\epsilon=1$. First note that
    \begin{align*} \| g^P_\mathcal{U}-g_\mathcal{U}^{\rm Q} \|_{L^2(\mathcal{U})}  & \leq \| g^P_\mathcal{U}-g_{\mathcal{U}'}^{\rm P} \|_{L^2(\mathcal{U})} +\| g_{\mathcal{U}'}^{\rm P}-g_{\mathcal{U}'}^{\rm Q} \|_{L^2(\mathcal{U})} +  +\| g_{\mathcal{U}}^{\rm Q}-g_{\mathcal{U}'}^{\rm Q} \|_{L^2(\mathcal{U})} \\
    & \leq {\rm diam}(\Omega) \left( \|g^P_\mathcal{U}-g_{\mathcal{U}'}^{\rm P}  \|_{\infty}+ \|g_{\mathcal{U}}^{\rm Q}-g_{\mathcal{U}'}^{\rm Q}  \|_{\infty} \right) + \| g_{\mathcal{U}'}^{\rm P}-g_{\mathcal{U}'}^{\rm Q} \|_{L^2(\mathcal{U})}.
    \end{align*}
    Using \cite[Theorem 4.5]{Barrio2022AnIC}, we obtain that 
    $$ \|g^P_\mathcal{U}-g_{\mathcal{U}'}^{\rm P}  \|_{\infty} \leq \| \mathcal{U}-\mathcal{U}'\|_s $$
    $$ \|g^Q_\mathcal{U}-g_{\mathcal{U}'}^{\rm Q}  \|_{\infty} \leq \| \mathcal{U}-\mathcal{U}'\|_s .$$
    The last term of the bound can be written as 
    \begin{align*} 
    	& \| g_{\mathcal{U}'}^{\rm P}-g_{\mathcal{U}'}^{\rm Q} \|_{L^2(\mathcal{U})}  =   \left( \int (g_{\mathcal{U}'}^{\rm P}(\x)-g_{\mathcal{U}'}^{\rm Q}(\x))^2 d(\mathcal{U}-\mathcal{U}')(\x) + \int (g_{\mathcal{U}'}^{\rm P}(\x)-g_{\mathcal{U}'}^{\rm Q}(\x))^2 d{\mathcal{U}'}(\x)  \right)^{\frac{1}{2}
    } \\
     & \leq \left( \int (g_{\mathcal{U}'}^{\rm P}(\x)-g_{\mathcal{U}'}^{\rm Q}(\x))^2 d(\mathcal{U}-\mathcal{U}')(\x) + \int (g_{\mathcal{U}'}^{\rm P}(\x)-g_{\mathcal{U}'}^{\rm Q}(\x))^2 d{\mathcal{U}'}(\x)  \right)^{\frac{1}{2}} \\
      & \leq \left( \left|\int \frac{(g_{\mathcal{U}'}^{\rm P}(\x)-g_{\mathcal{U}'}^{\rm Q}(\x))^2}{\| (g_{\mathcal{U}'}^{\rm P}-g_{\mathcal{U}'}^{\rm Q})^2\|_{\mathcal{C}^s(\Omega)}}  \| (g_{\mathcal{U}'}^{\rm P}-g_{\mathcal{U}'}^{\rm Q})^2\|_{\mathcal{C}^s(\Omega)} d(\mathcal{U}-\mathcal{U}')(\x)\right| + \int (g_{\mathcal{U}'}^{\rm P}(\x)-g_{\mathcal{U}'}^{\rm Q}(\x))^2 d{\mathcal{U}'}(\x)  \right)^{\frac{1}{2}} \\
     &  \leq  \left( \| (g_{\mathcal{U}'}^{\rm P}-g_{\mathcal{U}'}^{\rm Q})^2\|_{\mathcal{C}^s(\Omega)}  \sup_{f \in {\mathcal{C}^s(\Omega)}} \left| \int f(\x) d(\mathcal{U}-\mathcal{U}')(\x)\right| + \int (g_{\mathcal{U}'}^{\rm P}(\x)-g_{\mathcal{U}'}^{\rm Q}(\x))^2 d{\mathcal{U}'}(\x)  \right)^{\frac{1}{2}}.
     \end{align*}
    Now note on the first hand that $$ \sup_{f \in \mathcal{C}^s(\Omega)} \left| \int f(\x) d(\mathcal{U}-\mathcal{U}')(\x)\right| \leq \| \mathcal{U}-\mathcal{U}'\|_s. $$
    On the other hand recall that 
    $$ \|(g_{\mathcal{U}'}^{\rm P}-g_{\mathcal{U}'}^Q)^2 \|_{\mathcal{C}^s(\Omega)}= \sum_{i=0}^s\sum_{|\balpha|= i}\|D^{\balpha} (g_{\mathcal{U}'}^{\rm P}-g_{\mathcal{U}'}^Q)^2 \|_{\infty}, $$ 
    with the same notation $D^{\balpha}$ as in Section \ref{subsection:definition:notation}. 
    
    But for $|\balpha| \geq 1$, $D^{\balpha} (g_{\mathcal{U}'}^{\rm P}-g_{\mathcal{U}'}^Q)^2$ is a linear combination of product of derivatives of $g_{\mathcal{U}'}^{\rm P}-g_{\mathcal{U}'}^Q$, which enables to write that $$\sum_{i=0}^s \sum_{|\balpha|= i} \| D^{\balpha} (g_{\mathcal{U}'}^{\rm P}-g_{\mathcal{U}'}^Q)^2 \|_\infty \leq \sum_{i=0}^s \sum_{|\balpha|= i} \|  P_{\balpha} (g_{\mathcal{U}'}^{\rm P}-g_{\mathcal{U}'}^Q,\partial_1 (g_{\mathcal{U}'}^{\rm P}-g_{\mathcal{U}'}^Q),\dots,D^{\balpha}(g_{\mathcal{U}'}^{\rm P}-g_{\mathcal{U}'}^Q))  \|_{\infty}$$ for $P_{\balpha}$ polynomial functions.
     Since all functions are continuous and evaluated on a compact set $\Omega$, their supremum norm is bounded, which enables to write that $$\| D^{\balpha} (g_{\mathcal{U}'}^{\rm P}-g_{\mathcal{U}'}^Q)^2 \|_{\mathcal{C}^s(\Omega)} \leq  C_{\balpha}(\Omega) \| g_{\mathcal{U}'}^{\rm P}-g_{\mathcal{U}'}^Q\|_{\mathcal{C}^s(\Omega)}  $$ for a constant $C_{\balpha}(\Omega)$ which depends on $P,Q$ and $\balpha$ and the choice of $\Omega$.
    Since
    $$ \| (g_{\mathcal{U}'}^{\rm P}-g_{\mathcal{U}'}^{\rm Q})^2\|_{\mathcal{C}_s(\Omega)} \leq C_{\balpha}(\Omega) \|{\rm P}-{\rm Q}\|_s  $$
    we obtain that
    \begin{align*} \| g_{\mathcal{U}'}^{\rm P}-g_{\mathcal{U}'}^{\rm Q} \|_{L^2(\mathcal{U})} & \leq \left(\| \mathcal{U}-\mathcal{U}'\|_s  \|{\rm P}-{\rm Q}\|_s
    + \int (g_{\mathcal{U}'}^{\rm P}(\x)-g_{\mathcal{U}'}^{\rm Q}(\x))^2 d{\mathcal{U}'}(\x)
    \right)^{\frac{1}{2}} \\
    & \leq \| g_{\mathcal{U}'}^{\rm P}-g^Q_{\mathcal{U}'} \|_{L^2(\mathcal{U}')} + \left( \| \mathcal{U}-\mathcal{U}'\|_s  \|{\rm P}-{\rm Q}\|_s\right)^{1/2},\end{align*}
    which gives the inequality
    \begin{align*}
        & \| g_{\mathcal{U}'}^{\rm P}-g_{\mathcal{U}'}^{\rm Q} \|_{L^2(\mathcal{U})} - \| g_{\mathcal{U}'}^{\rm P}-g^Q_{\mathcal{U}'} \|_{L^2(\mathcal{U}')} \\
        & \leq   \left( \| \mathcal{U}-\mathcal{U}'\|_s  \|{\rm P}-{\rm Q}\|_s\right)^{1/2} + 2 {\rm diam}(\Omega) \| \mathcal{U}-\mathcal{U}'\|_s.
    \end{align*}
    Finally by symmetry, we obtain that
    $$ | \| g_{\mathcal{U}}^{\rm P}-g_{\mathcal{U}}^{\rm Q} \|_{L^2(\mathcal{U})} - \| g_{\mathcal{U}'}^{\rm P}-g^Q_{\mathcal{U}'} \|_{L^2(\mathcal{U}')} |\leq  \left( \| \mathcal{U}-\mathcal{U}'\|_s  \|{\rm P}-{\rm Q}\|_s\right)^{1/2}+2 {\rm diam}(\Omega) \| \mathcal{U}-\mathcal{U}'\|_s,$$
    which proves the result.
    \end{proof}


\section{Additional content for Section \ref{s:GP}}

\subsection{Likelihood in regression} \label{subsubsection:likelihood:regression}

In regression, we consider a data set $\mu_1 , y_1 , \ldots , \mu_n , y_n$, with $y_i = Z(\mu_i)$ where $Z$ is a centered GP with covariance function in $\{ K_{\btheta,\u}\}$. We let $U_n$ be the list $(\mu_1 , \ldots , \mu_n)$ and write $\K_{\btheta,\u}(U_n , U_n)$ for the $n \times n$ matrix with component $i,j$ equal to $K_{\btheta}(\mu_i , \mu_j)$. We also write ${\Y}_n$ for the $n \times 1$ vector $(y_1,\ldots , y_n)^\top$. Then the likelihood function is $g_{\mathcal{N}}(\Y_n,\boldsymbol{0},\K_{\btheta,\u}(U_n , U_n))$, where for any vectors $\m$ and $\x$ and matrix $\boldsymbol{\Sigma}$, in dimension $n$,
\begin{equation} \label{eq:gaussian:density}
g_{\mathcal{N}}(\x,\m,\boldsymbol{\Sigma})
=
\frac{1}{(2 \pi)^{n/2} \sqrt{\det(\boldsymbol{\Sigma})}}
e^{
- \frac{1}{2}
(\x - \m)^\top 
\boldsymbol{\Sigma}^{-1}
(\x - \m)
} 
\end{equation}
is the Gaussian density at $\x$ with mean $\m$ and covariance $\boldsymbol{\Sigma}$. Then $\btheta,\u$ can be selected by maximizing the likelihood function. Note that in regression, one can also use cross validation to estimate $\btheta$ and $\u$ \cite{rasmussen2006gaussian,Bachoc2013cross,zhang2010kriging}.

\subsection{Likelihood in classification} \label{subsubsection:likelihood:classification}

In classification, $Z$ is as before and
we consider a data set $\mu_1 , y_1 , \ldots , \mu_n , y_n$, where, conditionally to $Z$, $y_1,\ldots,y_n$ are independent with, for $i=1,\ldots,n$, 
\[
\mathbb{P}(y_i = 1) = 1- \mathbb{P}(y_i = 0)
=
\frac{e^{Z(\mu_i)}}{1+e^{Z(\mu_i)}}. 
\]

Then, from for instance Equation 3.30 in \cite{rasmussen2006gaussian}, the likelihood function is 
\begin{flalign*}
&
\int_{\mathbb{R}^n}
g_{\mathcal{N}}
\left( \v , \boldsymbol{0} , \K_{\btheta,\u}(U_n , U_n)  \right)
g(\Y_n|\v)
	d \v,
\end{flalign*}
with the density of $\Y_n$ given $(Z(\mu_1),\ldots,Z(\mu_n)) = \v$:
\begin{equation} \label{eq:lik:y:given:GP}
g(\Y_n|\v)
=
	\prod_{i=1}^n
	\left( 
	\left( 
	\frac{e^{v_i}}{1+e^{v_i}}
	\right)^{y_i}
	+
	\left(
	\frac{1}{1+e^{v_i}}
	\right)^{1-y_i}
	\right).
\end{equation}
 Above, $g_{\mathcal{N}}$ is as in \eqref{eq:gaussian:density}.
 
\subsection{Discussion of microergodicity}

 For both regression and classification, a natural theoretical question is the consistency of estimators for $\btheta$ and $\u$ as $n \to \infty$.
This question is essentially open for distributional inputs, as only a few results exist \cite{bachoc2017gaussian}. In contrast, most existing results address standard vector inputs \cite{stein1999interpolation,zhang2004inconsistent,bachoc2014asymptotic}.
A necessary condition for this consistency is that  $\btheta$ and $\u$ are microergodic, which means that changing them always changes the Gaussian measure of $Z$ on the set of functions from the input space to $
\mathbb{R}$. We refer to \cite{stein1999interpolation,bachoc2020gaussian} for more formal details. Related to our setting, \cite{bachoc2020gaussian} shows that microergodicity typically holds when the input space is a Hilbert ball. This result
provides positive indications that $\btheta$ and $\u$ may be microergodic in fairly general frameworks.
 
 \subsection{Prediction}
 
 We now aim at predicting a new output, associated to a new measure $\mu_0$, based on $y_1,\ldots,y_n$, that is to compute the conditional distribution of the new output given $y_1,\ldots,y_n$.

First, consider regression, where the output is $Z(\mu_0)$ and $y_1,\ldots,y_n$ are as in Section \ref{subsubsection:likelihood:regression}.
The conditional mean of $Z(\mu_0)$  given $y_1,\ldots,y_n$ is
\begin{align} \label{eq:pred:reg:mean}
& \mathbb{E}_{\btheta,\u}
\left( 
\left.
Z(\mu_0) 
\right| 
Z(\mu_1) , \ldots , Z(\mu_n)
\right)
=  \K_{\btheta,\u}(\mu_0,U_n)
\K_{\btheta,\u}(U_n,U_n)^{-1} {\Y}_n,
\end{align}
where $\K_{\btheta,\u}(\mu_0,U_n)$ is the $1 \times n$ vector with component $i$ equal to $K_{\btheta,\u}(\mu_0 , \mu_i)$, $i=1,\ldots,n$.
Thus, classically, GP prediction in regression consists in the conditional mean (also the $L^2$ projection).
We also have the well-known error indicator (conditional variance)
\begin{align} \label{eq:pred:reg:var}
&
\mathrm{var}_{\btheta,\u}
\left( 
\left.
Z(\mu_0) 
\right| 
Z(\mu_1) , \ldots , Z(\mu_n)
\right)
= \\ \notag
&
\K_{\btheta,\u}(\mu_0 , \mu_0)
-
\K_{\btheta,\u}(\mu_0,U_n)
\K_{\btheta,\u}(U_n,U_n)^{-1} 
\K_{\btheta,\u}(U_n,\mu_0),
\end{align}
where we let $\K_{\btheta,\u}(U_n,\mu_0) = K_{\btheta,\u}(\mu_0 , U_n)^\top$.

Second, consider classification, where the output is $y_0$, such that conditionally to $Z$, $y_0$ is independent from $y_1,\ldots,y_n$ (defined as in Section \ref{subsubsection:likelihood:classification}) and $\mathbb{P}(y_0 = 1) = 1- \mathbb{P}(y_0 = 0)
=
e^{Z(\mu_0)} /  (1+e^{Z(\mu_0)})$. Then, as follows from instance from Equations 3.9 and 3.10 in \cite{rasmussen2006gaussian}, the conditional probability that $y_0 = 1$ given $y_1 , \ldots , y_n$ is given by
\begin{flalign*}
\frac{1}{\kappa}
\int_{\mathbb{R}^{n+1}}
g_{\btheta,\u}(\v|\Y_n)
	g_{\btheta,\u}(z|\v)
	\frac{e^z}{1+e^z}
	d\v dz.
\end{flalign*}
Above, $g_{\btheta,\u}(z|\v)$ is the Gaussian density of $Z(\mu_0)$ at $z$ given $Z(\mu_i) = v_i$, $i=1,\ldots,n$, as obtained from \eqref{eq:pred:reg:mean} and \eqref{eq:pred:reg:var}, $\kappa = \int_{\mathbb{R}^n} g_{\btheta,\u}(\v|\Y_n) d \v$ and
\begin{flalign*}
& g_{\btheta,\u}(\v|\Y_n)
=
g_{\mathcal{N}}
\left( \v , 0 , \K_{\btheta,\u}(U_n , U_n) \right)
	g(\Y_n|\v),
\end{flalign*}
with $g_{\mathcal{N}}$ as in \eqref{eq:gaussian:density}.

\section{Proofs for Section \ref{s:GP}}


    \begin{proof}[Proof of Proposition \ref{proposition:continuity}]
    Let $d_Z$ be the canonical distance on $\mathcal{P}_{\delta}$ of the covariance function in \eqref{eq:kernel}, given by, for ${\rm P},{\rm Q} \in \mathcal{P}_{\delta}$,
    \[
    d_Z({\rm P},{\rm Q}) 
    =
    \sqrt{
    2F(0) - 2 F \left( 
    \|g^{\rm P}-g^{\rm Q}\|_{L^2(\mathcal{U})}
    \right).
    }
    \]
    For $\epsilon >0$, let $\mathcal{N}(\epsilon,\mathcal{P}_{\delta},d_Z)$ be the minimum number of $d_Z$-balls in $\mathcal{P}_{\delta}$ of radius $\epsilon$ needed to cover $\mathcal{P}_{\delta}$. From for instance Theorem 1.1 in \cite{adler1990introduction}, in order to conclude the proof, it is sufficient to show that 
    \begin{equation} \label{eq:integral:entropy:finite}
    \int_{0}^{\infty} 
    \sqrt{ \log(\mathcal{N}(\epsilon,\mathcal{P}_{\delta},d_Z) ) } 
    d \epsilon 
    < \infty. 
    \end{equation}
    
    Let $\alpha_d = \left\lceil d / a \right\rceil+2 $ and  $s_d = d \alpha_d$. Proposition \ref{Lemma:cuadratic} yields that for ${\rm P},{\rm Q} \in \mathcal{P}_{\delta}$,
    \[
    \|g^{\rm P}-g^{\rm Q}\|_{L^2(\mathcal{U})} 
    \leq B_d 
    \| {\rm P} - {\rm Q} \|_{s_d},
    \]
    where $\| \cdot \|_{s_d}$ is defined in \eqref{eq:norme:P:Q:C:s:un} and $B_d$ is a constant not depending on $\epsilon$. For any ${\rm P}$ and ${\rm Q}$ in $\mathcal{P}_{\delta}$, with densities $p$ and $q$, we have, with $\| \cdot \|_{\mathcal{C}^{s_d}(\Omega)}$ defined in \eqref{eq:norm:holder:alpha},
    \begin{align} \label{eq:dX:bound}
        d_Z({\rm P},{\rm Q}) 
        & \leq 
        \sqrt{ 
       2 A \|g^{\rm P}-g^{\rm Q}\|_{L^2(\mathcal{U})}^{a}
        } \notag \\
        & \leq 
         \sqrt{ 
       2 A  B_d^a 
    \| {\rm P} - {\rm Q} \|_{s_d}^a
        } 
        \notag
        \\
        & =
         \sqrt{ 
       2 A  B_d^a}
        \sup_{f \in \mathcal{C}^{s_d}(\Omega) , \|f\|_{\mathcal{C}^{s_d}(\Omega)} \leq 1 }
        \left(
        \int_{\Omega} 
        f(\x) 
        (p(\x) - q(\x))
        d \x
        \right)^{a/2}.
    \end{align}

    For $f \in  \mathcal{C}^{s_d}(\Omega) , \|f\|_{\mathcal{C}^{s_d}(\Omega)} \leq 1$, we can multiply $f$ by an infinitely differentiable function that is zero on
    $\{ \bt \in \Omega, d(\bt,\partial \Omega) \leq b/2 \}$, and one on $\{ \bt \in \Omega, d(\bt,\partial \Omega) \geq b \}$ (that exists by Lemma \ref{lemma:exists:smooth:function:boundary}). Let us write $\tilde{f}$ the result of this multiplication. Since ${\rm P}$ and ${\rm Q}$ above are in $\mathcal{P}_{\delta}$, we have
    \begin{equation} \label{eq:int:Omega:f:tilde:f}
    \int_{\Omega} 
        f(\x) 
        (p(\x) - q(\x))
        d \x
        =
        \int_{\Omega} 
        \tilde{f}(\x) 
        (p(\x) - q(\x))
        d \x.
    \end{equation}

    By taking the infinitely differentiable function the same for each $f$, we obtain  $\|\tilde{f}\|_{\mathcal{C}^{s_d}(\Omega)} \leq D_d$, where $D_d$ is a constant.
    
    Now we consider a bounded compact hyper-rectangle $R$ such that $\Omega$ belongs to the interior of $R$. Above, $p$ and $q$ are summable and continuous on $\Omega$ and are zero on $\{ \x \in \Omega , d(\x , \partial \Omega) \leq b \}$, so we can extend them to summable continuous functions on $R$, that take the value $0$ on $R \backslash \Omega$. Let us also extend $\tilde{f}$ on $R$ by taking values zero on $R \backslash \Omega$. We then have $\|\tilde{f}\|_{\mathcal{C}^{s_d}(R)} \leq D_d$, by defining $\| \cdot \|_{\mathcal{C}^{s_d}(R)}$ as in \eqref{eq:norm:holder:alpha} (replacing $\Omega$ by $R$). 
    
    We can thus write 
    \begin{equation} \label{eq:int:Omega:equal:int:R}
     \int_{\Omega} 
        \tilde{f}(\x) 
        (p(\x) - q(\x))
        d \x 
    =
     \int_{R} 
        \tilde{f}(\x) 
        (p(\x) - q(\x))
        d \x, 
    \end{equation}
    where we use the same notation $p,q,\tilde{f}$ both for the original functions on $\Omega$ and their extensions on $R$. The function $\tilde{f}$ is $s_d$ times differentiable on $R$, with all the derivatives of order $s_d$ or less that cancel out on the boundary of $R$. 
    Write the hyper-rectangle $R$ as $\prod_{j=1}^d [ \ell_j ,u_j]$.
    Let
    for $i=0, \ldots, d$,
    $I^{(1,i)} q$ be the function defined on $R$ by, for $i=0$,  $I^{(1,0)} q = q$ and for $i \geq 1$, $(x_1 , \ldots , x_d) \in R$,
    \[
    (I^{(1,i)} q)(x_1,\ldots,x_d) =
    \int_{\ell_i}^{x_i} 
    (I^{(1,i-1)} q) (x_1,\ldots,x_{i-1},t,x_{i+1},\ldots,x_d) 
    dt.
    \]
    Let
    for $i=0, \ldots, d$,
    $I^{(2,i)} q$ be the function defined on $R$ by, for $i=0$,  $I^{(2,0)} q = I^{(1,d)}q$ and for $i \geq 1$,  $(x_1 , \ldots , x_d) \in R$,
    \[
    (I^{(2,i)} q)(x_1,\ldots,x_d) =
    \int_{\ell_i}^{x_i} 
    (I^{(2,i-1)} q) (x_1,\ldots,x_{i-1},t,x_{i+1},\ldots,x_d)
    dt.
    \]
    We iterate like that until defining $I^{(\alpha_d,d)} q $ from $R$ to $\mathbb{R}$ that satisfy $D^{(\alpha_1 , \ldots , \alpha_d)} I^{(\alpha_d,d)} q = q$. 
    We define $I^{(\alpha_d,d)} p$ similarly.  
    
    Hence, we can apply multi-dimensional integration by part on $R$ to obtain 
     \begin{align*}
     \int_{R} 
        \tilde{f}(\x) 
        (p(\x) - q(\x))
        d \x
    &    =
    (-1)^{d \alpha_d}
         \int_{R} 
        (D^{(\alpha_1 , \ldots , \alpha_d)}
        \tilde{f})(\x)
        \left(
        (I^{(\alpha_d,d)} p) (\x) 
        -
        (I^{(\alpha_d,d)}q)(\x)
        \right)
        d \x
        \\
    &    =
    (-1)^{d \alpha_d}
    \int_{\Omega} 
        (D^{(\alpha_1 , \ldots , \alpha_d)}
        \tilde{f})(\x)
        \left(
        (I^{(\alpha_d,d)} p) (\x) 
        -
        (I^{(\alpha_d,d)}q)(\x)
        \right)
        d \x.
    \end{align*}
    Hence going back to \eqref{eq:dX:bound}, \eqref{eq:int:Omega:f:tilde:f} and \eqref{eq:int:Omega:equal:int:R}, with $\ell_d$ denoting Lebesgue measure, we have
    \begin{align} \label{eq:dX:bounded:Ip:Iq}
    d_Z({\rm P},{\rm Q})
    & \leq 
    \sqrt{2 A B_d^a} 
      \sup_{f \in \mathcal{C}^{s_d}(\Omega) , \|f\|_{\mathcal{C}^{s_d}(\Omega)} \leq 1 } 
      \left(
    \int_{\Omega} 
        (D^{(\alpha_1 , \ldots , \alpha_d)}
        \tilde{f})(\x)
        \left(
        (I^{(\alpha_d,d)} p) (\x) 
        -
        (I^{(\alpha_d,d)}q)(\x)
        \right)
        d \x
    \right)^{a/2}
      \notag \\
      & \leq 
      \sqrt{2 A B_d^a }
      D_d^{a/2}
      \ell_d(\Omega)^{a/2}
      \left| 
      \left| 
       I^{(\alpha_d,d)} p
        -
        I^{(\alpha_d,d)}q
      \right|
      \right|_{\infty}^{a/2}.
    \end{align}
    Since $p$ is a density function, we can show by induction that we have, for any $\beta_1, \ldots , \beta_d \in \mathbb{N}$ with $ \beta_1 \leq \alpha_1-1, \ldots , \beta_d \leq \alpha_d-1$,
    $ \| D^{(\beta_1,\ldots,\beta_d)} I^{(\alpha_d,d)} p \|_{\infty} \leq
    \max(1 , \max_{j=1,\ldots,d} (u_j - \ell_j) )^{d(\alpha_d - 1)} $.  Let $E_d = \max(1 , \max_{j=1,\ldots,d} (u_j - \ell_j) )^{d(\alpha_d - 1)} $.

    Define the space $\mathcal{C}_{E_d}^{\alpha_d-1}(\Omega)$ as the ball with the norm $\| \cdot \|_{\mathcal{C}^{\alpha_d-1}(\Omega)}$ given by \eqref{eq:norm:holder:alpha}, with center $0$ and radius $E_d$. For $\epsilon >0$, consider a $\epsilon$-covering of this ball, with norm $\| \cdot \|_{\infty}$, with cardinality $N$. From Theorem 2.7.1 in \cite{vandervaart2013weak}, we can select $N$ such that
    \[
    \log(N)
    \leq 
    F_d \epsilon^{-d/(\alpha_d-1)},
    \]
    with a constant $F_d$ that does not depend on $\epsilon$. 
    For each of the $N$ balls that contains one function of the form $I^{(\alpha_d,d)} q $ where $q$ is the density of some ${\rm Q} \in \mathcal{P}_{\delta}$, we consider such a function  $I^{(\alpha_d,d)} q $. There are $N'$ such functions that we write $I^{(\alpha_d,d)} q_1 , \ldots , I^{(\alpha_d,d)} q_{N'}$. For each ${\rm P} 
    \in \mathcal{P}_{\delta}$ with density $p$, $I^{(\alpha_d,d)} p$ belongs to 
    $\mathcal{C}^{\alpha_d-1}_{E_d}$
    and thus belongs to the same ball as some  $I^{(\alpha_d,d)} q_i$ with $i \in \{1 , \ldots , N' \}$ and thus $\| I^{(\alpha_d,d)} p - I^{(\alpha_d,d)} q_i\|_{\infty} \leq 2 \epsilon$. Hence from \eqref{eq:dX:bounded:Ip:Iq} we have, whith $Q_i \in \mathcal{P}_{\delta}$ having density $q_i$, 
    \[
    d_Z({\rm P} , Q_i ) 
    \leq 
    \sqrt{ 2 A B_d^a}
    D_d^{a/2}
    \ell_d(\Omega)^{a/2}
    (2 \epsilon)^{a/2}. 
    \] 
    Hence, there are constants $G_d,H_d$ such that for $0 < t \leq 1$,
    \[
    \mathcal{N}(t,\mathcal{P}_{\delta},d_Z) 
    \leq G_d 
    e^{H_d t^{- 2d / a (\alpha_d-1)}}. 
    \]
    Since $d / a (\alpha_d - 1) <1$, we thus obtain that \eqref{eq:integral:entropy:finite} holds. 
    \end{proof}

    \begin{lemma} \label{lemma:exists:smooth:function:boundary}
    Let $\Omega$ be compact and let $b >0$.
    There exists an infinitely differentiable function that is zero on
    $\{ \bt \in \Omega, d(\bt,\delta \Omega) \leq b/2 \}$, and one on $\{ \bt \in \Omega, d(\bt,\delta \Omega) \geq b \}$.
    \end{lemma}

    \begin{proof}[Proof of Lemma \ref{lemma:exists:smooth:function:boundary}]
    Let $g$ be an infinitely differentiable function with integral one and which support is included in the Euclidean ball of $\mathbb{R}^d$ with center $0$ and radius $b/4$. Let for $r \geq 0$, $\Omega_r = \{ \bt \in \Omega, d(\bt,\delta \Omega) \geq r \}$. Consider the function $h$ on $\mathbb{R}^d$ defined by, for $\bt \in \mathbb{R}^d$,
    \[
    h( \bt) = \int_{\mathbb{R}^d} 
    \mathbf{1}_{\{ \x \in \Omega_{3b/4} \}}
    g(\x- \bt) d\x.
    \]
    Then $h$ is infinitely differentiable by dominated convergence. For $\bt \in \Omega_{b}$ and $\x$ such that $\|\bt-\x\| \leq b/4$, then $ \x \in \Omega_{3b/4}$. Hence  
    \[
    h(\bt) = \int_{\mathbb{R}^d} 
    g(\x-\bt) d\x =1.
    \]
    For $\bt \in \Omega$ with $d(\bt , \delta \Omega) \leq b/2$, and $\x$ such that $\|\bt-\x\| < b/4$, then $ d(\x , \delta \Omega) < 3b/4$. Hence 
    \[
    h(\bt) = \int_{\mathbb{R}^d} 
    0 d\x  =0.
    \]
    This concludes the proof.
    \end{proof}


\section{Details of the algorithm}
\subsection{Kernel}

We use the kernel:

\begin{equation}
    \label{eq:rbfsink}
    K(\rm P,\rm Q)=l\exp{-\frac{\|g^{\rm P}_{\u}-g^{\rm Q}_{\u}\|}{2\sigma^2}}.
\end{equation}

Here the parameters are the tuple $\btheta=(l,\sigma)$ where $l\in\mathbb{R}$ is the length scale and $\sigma\in\mathbb{R}$ the scalar variance.

For simplicity we only train gaussian process with zero mean function $\mu=\bf 0$. This does not prevent the GP to reach satisfying level of RMSE/accuracy as illustrated in experiments.  
  
The experiments based on Radial Basis Function (RBF) kernel uses a similar form:

\begin{equation}
    \label{eq:rbf}
    K_{\text{RBF}}(x,y)=l\exp{-\frac{\|x-y\|}{2\sigma^2}}.
\end{equation}

\subsection{Sinkhorn's algorithm}

Sinkhorn's algorithm is an iterative algorithm that take advantage of \textit{approximately good} solutions. Hence, the dual variables are re-used from one step of optimization to the other. When the steps are small, it guarantees that the initialization is not far away from the optimum. It allows the algorithm to benefit from a significant speed-up.  

\subsection{L-BFGS}

We apply Limited Memory  Broyden–Fletcher–Goldfarb–Shanno algorithm (L-BFGS), which is an order 2 method, to enjoy faster convergence than order-1 methods such as Gradient Descent. The dominant cost of the algorithm is induced by the size of the support $\u$ and by the dimension of the points $\x_i\in\mathbb{R}^d$ since $(\sigma,l)\in\mathbb{R}^2$. The total dimension of search space is hence $nd+2$.   
  
We select the optimal stepsize at each iteration with a \textit{zoom line search} (Algorithm 3.6 of~\cite{nocedal1999numerical}, pg. 59-61. Tries cubic, quadratic, and bisection methods of zooming).  

\subsection{Runtime cost against MMD}

The Maximum Mean Discrepancy (MMD) is based on RBF:

\begin{equation}
    \label{eq:mmd}
    \text{MMD}(\rm P,\rm Q)=\mathbb{E}_{\rm P}(K_{\text{RBF}}(X,X))+\mathbb{E}_{\rm Q}(K_{\text{RBF}}(Y,Y))-2\mathbb{E}_{\rm P,\rm Q}(K_{\text{RBF}}(X,Y)).
\end{equation}

The MMD distance is turned into a kernel with additional parameter $\hat\sigma$:

\begin{equation}
    \label{eq:mmdker}
    K_{\text{MMD}}(\rm P,\rm Q)=\hat \sigma\exp{(-\text{MMD}^2(\rm P,\rm Q))}.
\end{equation}

The kernel in equation~\ref{eq:mmdker} is universal (see Theorem 2.2 of~\cite{NIPS2010_4e0cb6fb} for example).

For fair comparison Sinkhorn $\u$-Kernel and MMD kernel are benchmarked on the same hardware under `@jax.jit` compiled code to benefit from GPU acceleration.  
  
We report runtime results in table~\ref{tab:runtimecost}. The clouds all share the same coordinates (but not the same weights). The pairwise distances between points of the clouds are pre-computed to speed-up both MMD and Sinkhorn iterations. We notice that Sinkhorn takes advantage of pre-computing the low dimension embeddings in dimension $|\u|=6$, independent of the cloud size. We chose $\epsilon=10^{-2}$ as regularization parameter. The points $\u$ are sampled uniformly in $[0,1]^2$ square, while points from the clouds $P_i$ are a discretization of $[0,1]^2$ square with equally spaced coordinates.

\begin{table}[]
    \centering
    \begin{tabular}{|cc|ccc|}
    \hline
     Number of clouds & Cloud size & Sinkhorn with $|\u|=6$ & Sinkhorn with $|\u|=12$ & MMD\\
    \hline
    \hline
     $n=50$ & $m=100$ & 0.009s & 0.001s & \bf 0.001s \\
     $n=100$ & $m=100$ & 0.013s & 0.011s &\bf 0.005s  \\
     $n=100$ & $m=400$ & \bf 0.007s & 0.021s & 0.055s \\
     $n=400$ & $m=400$ & \bf 0.018s & 0.059s & 0.683s \\
     $n=400$ & $m=625$ & \bf 0.026s & 0.088s & 1.681s \\
     $n=1000$ & $m=625$ & \bf 0.064s & 0.147s & 10.834s \\
     $n=1000$ & $m=1000$ & \bf 0.090s & 0.158s & 14.207s \\
    \hline
    \end{tabular}
    \caption{Runtime cost of Sinkhorn $\u$-Kernel (ours) against Maximum Mean Discrepancy (MMD). The cost reported corresponds to the overall process: computation of regularized OT plan and of the kernel for Sinkhorn $\u$, and computation of MMD distance for MMD. Clouds are in dimension $d=2$.}
    \label{tab:runtimecost}
\end{table}


\section{Details of Numerical experiments}
\subsection{Visualizing dual variables}

In figure~\ref{fig:illustratedual} we introduce an example with two distributions $P$ and $Q$ obtained by taking finite sample from isotropic gaussians. For $P$ we sample $30$ points from $\mathcal{N}([-2,-2],0.4)$ and for $Q$ we sample $50$ points from $\mathcal{N}([-1,1],0.3)$. We chose for $\u$ a finite sample of size $120$ from the unit ball $\mathbb{B}({\bf 0},1)$.  
  
We plot both the distributions and the values taken by $g^P$ and $g^Q$ respectively, by sorting dual variables arbitrarily by increasing error of $|g^P_-g^Q_i|$.  

\begin{figure}
    \centering
    \includegraphics[scale=0.3]{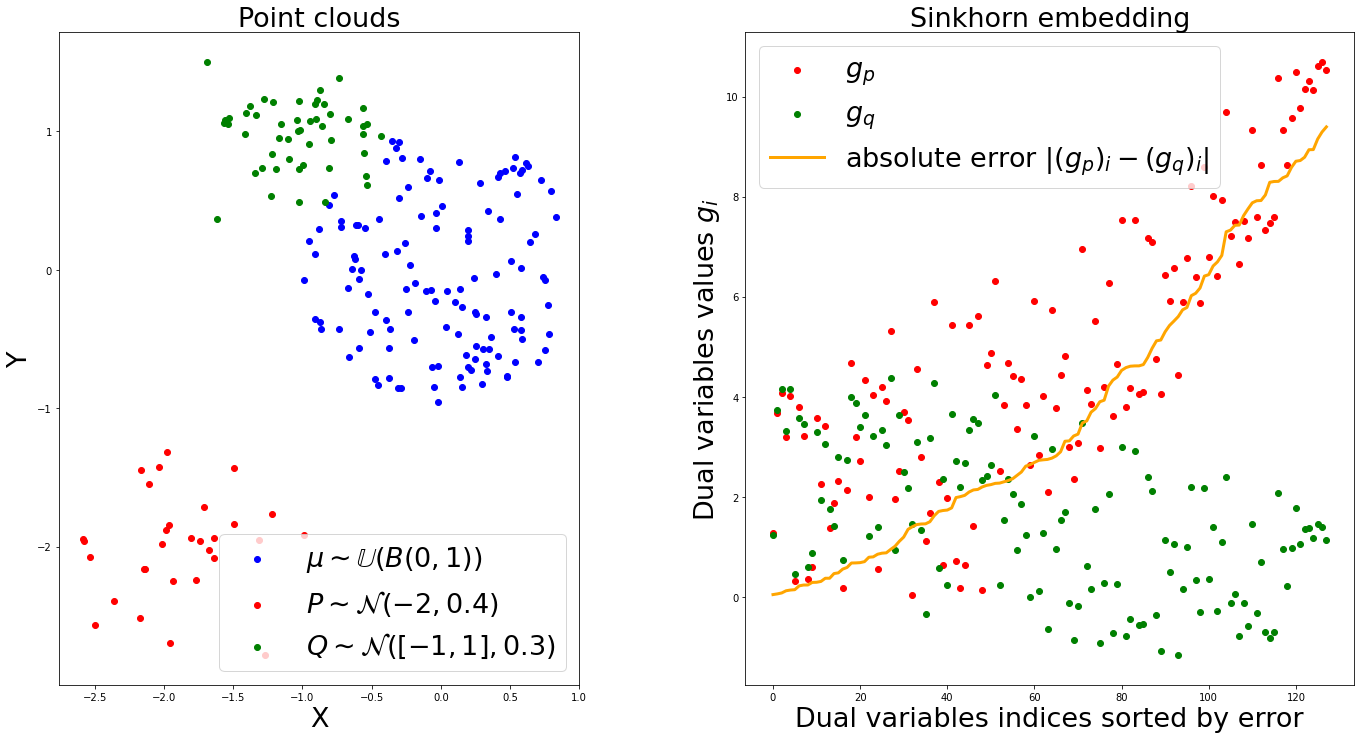}
    \caption{Vizualization of dual variables $g^P$ and $g^Q$. For $P$ we sample $30$ points from $\mathcal{N}([-2,-2],0.4)$ and for $Q$ we sample $50$ points from $\mathcal{N}([-1,1],0.3)$. We chose for $\u$ a finite sample of size $120$ from the unit ball $\mathbb{B}(\bf 0,1)$. }
    \label{fig:illustratedual}
\end{figure}

\subsection{Toy dataset}

All clouds are centered and rescaled so the overall dataset (obtained by merging all clouds) has zero mean and unit variance across all dimensions.   
  
We study a discretization of $\u$ in experiment of section~\ref{subsec:toy}. We chose $\u$ to be a discretization of input space $[0,1]^2$ as a $50\times 50$ grid. The density is chosen uniform over this discretization of $2500$ points. Hence each regularized optimal transporation plan is between a gaussian and an uniform measure over the square $[0,1]^2$. In this case the dual variable $g^P_{\u}$ can be vizualized as an image in definition $50\times 50$.   
  
For $20$ train examples, we plot the image $g^P_{\u}$ in figure~\ref{fig:toygrid}. We see that all those image appear ``blurry'' we show the role of regularization in OT. Moreover those images seem to correspond to a ``blob'' chose coordinates correspond with the one of the cloud $P_i$. This figure help to understand what the dual variables exactly look like in toy examples.  

\begin{figure}
    \centering
    \includegraphics[scale=0.4]{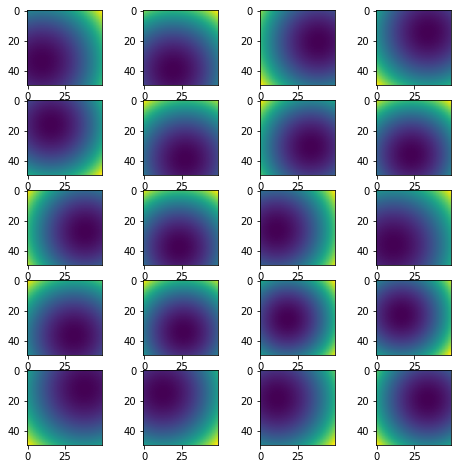}
    \caption{Plot of dual variables $g^P_{\u}$ for $5\times 4=20$ distributions $P_i$ from the toy example of section~\ref{subsec:toy}. The position of the center of each ``blob'' (i.e. mean of the gaussian) can be clearly seen by looking at the dual variables.}
    \label{fig:toygrid}
\end{figure}

\subsection{Mnist and Fashion-Mnist datasets}\label{sec:appmnist}

For RBF kernel, the image are normalized so that the pixel intensity lies in $[0,1]$ range.

The figure~\ref{fig:fashionmnist} illustrates the evolution of $x_i$'s and $w_i$'s for $\u$ in the case of an image of shoe from Fashion-Mnist.  
\begin{figure}
    \centering
    \includegraphics[width=1.\textwidth]{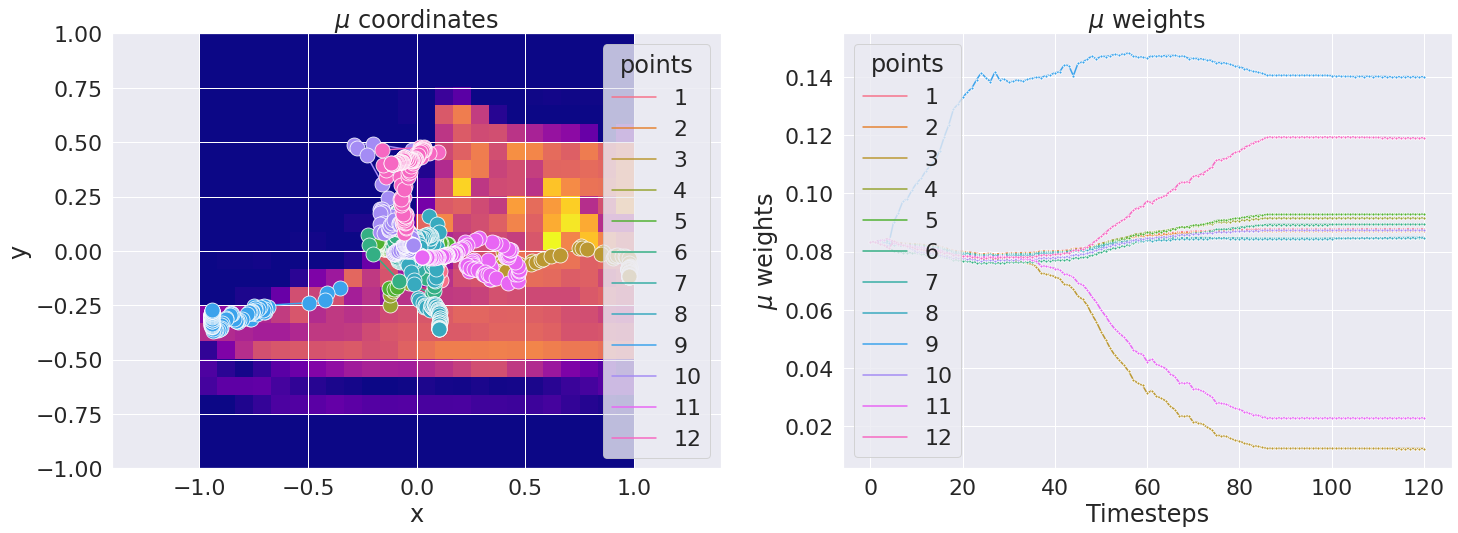}
    \caption{Evolution of $x_i$'s and $w_i$'s for $\u$ in the ``sneakers'' versus ``sandals'' task.}
    \label{fig:fashionmnist}
\end{figure}

\subsubsection{Sensivity to random affine transformations}
  
In figure~\ref{fig:randmnist} we plot a set of Mnist images on which random affine transformations have been applied. We follow the protocol of~\cite{meunier2022distribution} and we sample a translation uniformly at random in range $[-6,6]$ pixels, and a rotations uniformly at random in range $[-\frac{\pi}{3},\frac{\pi}{3}]$ rads.
  
\begin{figure}
    \centering
    \includegraphics[scale=0.4]{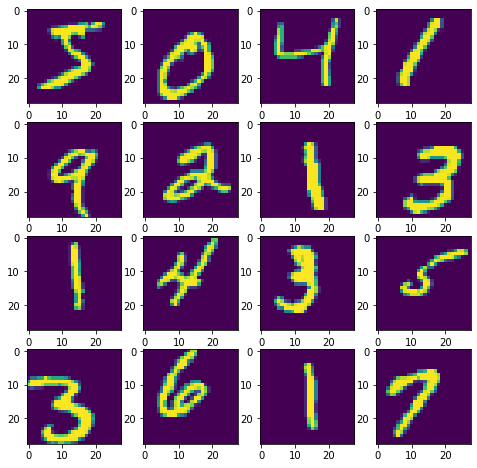}
    \caption{Mnist images with random affine transformations: translation uniformly at random in range $[-6,6]$ pixels.}
    \label{fig:randmnist}
\end{figure}

In figure~\ref{fig:muspace} we study the influence of random affine transformations in dual variable space $g^P_{\u}$, versus pixel space. In this experiment the measure $\u$ is chosen to have full support in dimension $28\times 28=784$. The measure is taken chosen uniform on pixel space. The images are process as clouds of $28\times 28=784$ pixels in dimension $2$. The regularization factor is chosen to be $\epsilon=10^{-2}$.  
  
We see that dual variables are less sensitive to translations than pixels. In the third row of figure~\ref{fig:muspace} the dual variables are modified in a way that hints the direction and amplitude of translation, whereas in fourth row the translation in pixel space has major consequences on the image and exhibit a huge euclidean norm. It shows that $\mu$ Sinkhorn dual variables are better tailored to handle translate than conventional euclidean metric, thanks to the properties of OT in translations.

\begin{figure}
    \centering
    \includegraphics[scale=0.4]{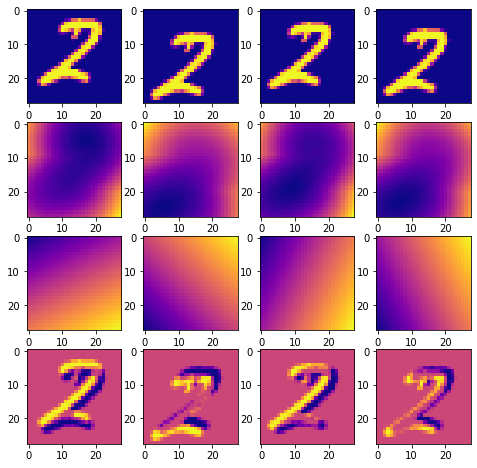}
    \caption{Visualization of translations on dual variables for an Mnist image for $|\u|=28\times 28=784$. \textbf{Top row}: original image $x$ with different affine transformations $\tilde x$. \textbf{Second row}: Dual variables of translated images $g(\tilde x)$. \textbf{Third row}: pixel-wise difference between the dual variables of original (non modified) image and translated images $g(x)-g(\tilde x)$. \textbf{Fourth row}: pixel-wise difference between original image and translated image $x-\tilde x$, in pixel space. We see that any translation has major impact in pixel space, but only mild consequences in dual variables space. Moreover the map $g(x)-g(\tilde x)$ hints the nature of the translation, whereas $x-\tilde x$ is harder to interpret.}
    \label{fig:muspace}
\end{figure}

\subsection{C-SVM results}

\subsubsection{UIUC dataset}\label{sec:appuiuc}

We report here the results of classification with C-SVM on the University of Illinois Urbana Champaign (UIUC) texture dataset~\cite{lazebnik2005sparse}, using the same protocol as~\cite{kolouri2016sliced}. Samples are shown in figure~\ref{fig:uiucsamples}. The dataset contains 25 different classes of texture on a total of $1000$ images (only $40$ images per class).  
  
The Gray Level Co-occurences Matrices (GLCM) is computed with Scikit-image library~\cite{van2014scikit}. The images are illustrated in figure~\ref{fig:glcm}.  
The $\gamma$ of the SVM is obtained by following the ``scale'' policy of Scikit-learn library, applied on normalized features. We apply a grid search in logspace on $C$ parameter of SVM, ranging from $10^{-1}$ to $10^3$. We optimal parameter is selected by selecting the highest average accuracy in $5$-fold cross validation.   
  
We compare the result against Radial Basis Function (RBF) kernel applied in raw (unprocessed pixel). The results are reported in Table~\ref{tab:uiucresult}. We see that RBF kernel is actually as good as any other approach, contrary to what was claimed in work of~\cite{kolouri2016sliced}.  
  
\begin{table}
    \centering
    \begin{tabular}{cccc}
         Dataset & $\u$-Sinkhorn (ours) & RBF & Sliced Wasserstein (\cite{kolouri2016sliced}) \\
         \hline
         \hline
        UIUC Textures & 87.2 & 87.3 & $88\pm 1$\\
        Mnist (1300 examples) & 92.50 & 92.46 & N/A\\
    \end{tabular}
    \caption{Validation accuracy of C-SVM with different kernels on GLCM embeddings of UIUC texture dataset~\cite{lazebnik2005sparse}, and $1300$ example of Mnist (10 classes), with $5$ folds cross-validation.}
    \label{tab:uiucresult}
\end{table}

\begin{figure}
    \centering
    \includegraphics[width=0.5\textwidth]{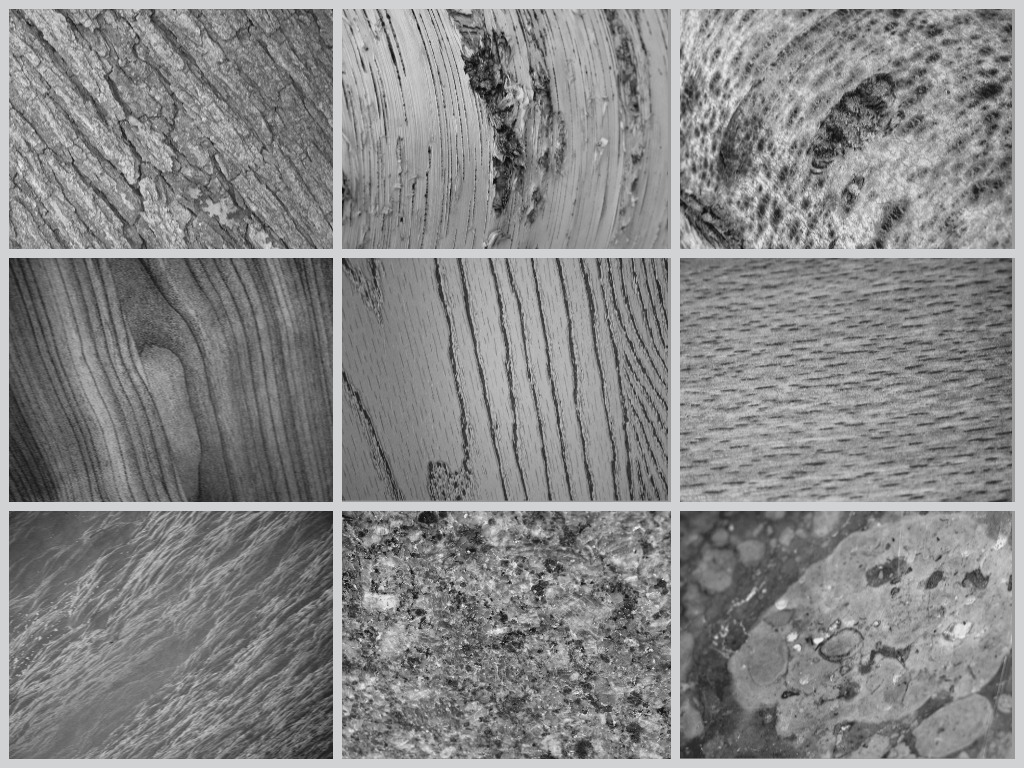}
    \caption{Random samples from UIUC texture dataset.}
    \label{fig:uiucsamples}
\end{figure}
  
\begin{figure*}
    \centering
    \includegraphics[width=0.5\textwidth]{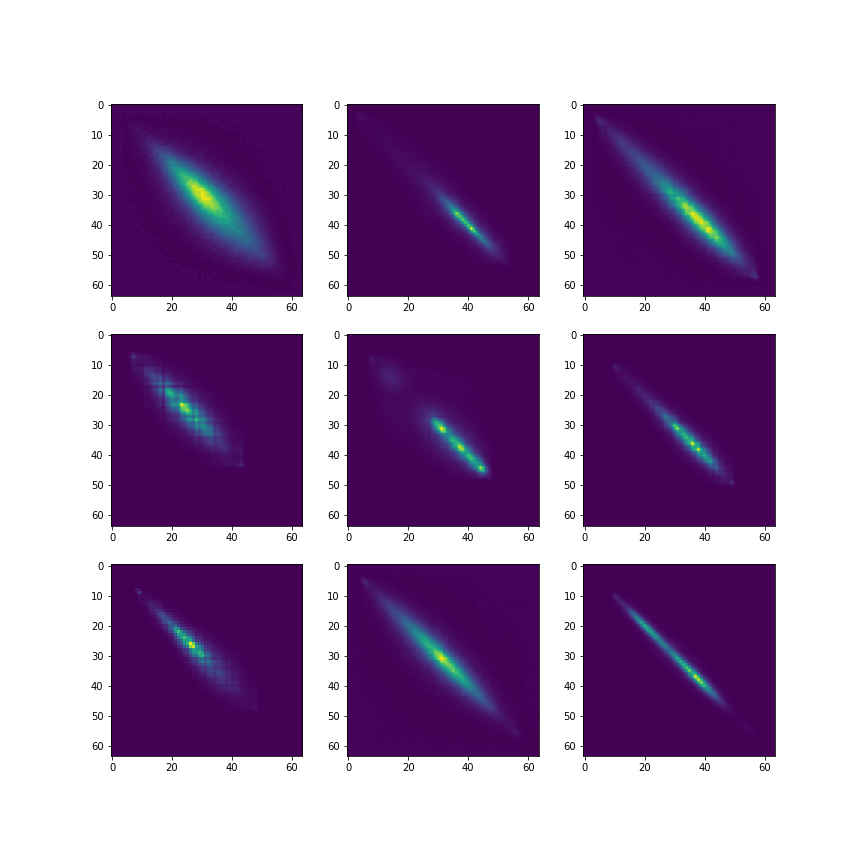}
    \caption{Gray Level Co-occurences Matrix (GLCM) of random samples of UIUC texture dataset. Can be re-normalized into a discrete 2D distribution.}
    \label{fig:glcm}
\end{figure*}

\subsubsection{Mnist C-SVM}

We chose a measure $\u$ with full support as in section~\ref{sec:appmnist}. We select $1300$ examples at random in Mnist train set with from all 10 classes, and we apply the protocol of section~\ref{sec:appuiuc}. The results of the best estimator found with $5$ fold cross-validation are reported on an independent test set of size $1000$ in table~\ref{tab:uiucresult}.  
  
It shows that the RBF kernel in pixel space is already very efficient and is barely outperformed by our approach based on distributions when $\u$ has full support.  
  
\subsection{Hardware and code}

All the experiments were run on the publicly available GPU Colab hardware.   
  
The code can be found on anonymous repository: \url{https://anonymous.4open.science/r/SinkhornMuGP-D37E/README.md}.   


\end{document}